\documentclass{article}

\usepackage{microtype}
\usepackage{graphicx}
\usepackage{subfigure}
\usepackage{booktabs} 

\usepackage{hyperref}

\usepackage[accepted]{icml2018}

\icmltitlerunning{Adaptive Exploration-Exploitation Tradeoff for Opportunistic Bandits}

\usepackage[english]{babel}  
\usepackage{color}
\usepackage{amsmath}
\usepackage{amsfonts}
\usepackage{amsthm}
\usepackage{amssymb}
\usepackage{bbm}
\usepackage{dsfont}
\usepackage{subfigure}
\usepackage{algorithm}
\usepackage{algorithmic}

\DeclareMathOperator*{\argmax}{arg\,max}

\allowdisplaybreaks[1]
\newtheorem{theorem}{Theorem}

\newtheorem{lemma}{Lemma}
\newtheorem{corol}{Corollary}

\begin{document}
\twocolumn[
\icmltitle{Adaptive Exploration-Exploitation Tradeoff for Opportunistic Bandits}




\begin{icmlauthorlist}
\icmlauthor{Huasen Wu}{to}
\icmlauthor{Xueying Guo}{goo}
\icmlauthor{Xin Liu}{goo}
\end{icmlauthorlist}

\icmlaffiliation{to}{Twitter Inc., San Francisco, California, USA;}
\icmlaffiliation{goo}{University of California, Davis, California, USA. This work was partially completed while the first author was a postdoctoral researcher at University of California, Davis}

\icmlcorrespondingauthor{Xin Liu}{xinliu@ucdavis.edu}

\icmlkeywords{Multi-armed bandit, adaptive exploration-exploitation, upper-confidence-bound}

\vskip 0.3in
]



\printAffiliationsAndNotice{}  
           
\begin{abstract}
In this paper, we propose and study opportunistic bandits - a new variant of bandits where the regret of pulling a suboptimal arm varies under different environmental conditions, such as network load or produce price. When the  load/price is low, 
 so is the cost/regret of pulling a suboptimal arm (e.g., trying a suboptimal network configuration). Therefore, intuitively, we could explore more when the load/price is low and exploit more when the load/price is high. Inspired by this intuition, we propose  an Adaptive Upper-Confidence-Bound (AdaUCB) algorithm to adaptively balance the exploration-exploitation tradeoff for opportunistic bandits.  We prove that AdaUCB achieves $O(\log T)$ regret with
a {\em smaller coefficient} than the traditional UCB algorithm.
Furthermore, AdaUCB achieves $O(1)$ regret with respect to $T$ if the exploration cost is zero when the load level is below  a certain threshold. Last, based on both synthetic data and real-world traces, experimental results show that AdaUCB significantly outperforms other bandit algorithms, such as UCB and TS (Thompson Sampling), under large load/price fluctuations.
\end{abstract}

\section{Introduction} \label{sec:intro}


In existing studies of multi-armed bandits (MABs) \cite{Auer2002ML:UCB,Bubeck2012Survey}, pulling a suboptimal arm results in a constant regret. While this is a valid assumption in many existing applications,  there exists a variety of applications where the actual regret of pulling a suboptimal arm may vary depending on external conditions. Consider the following application scenarios. 

%

\textbf{Motivating scenario 1: price variation.}  MAB has been widely used in studying effective procedures and treatments \cite{Lai1987adaptive,Press2009NACt,Villar2015SS:MAB4Clinical},  including in agriculture.  In agriculture, price often varies significantly for produce and livestock. For example, the pork price varied from \$0.46/lb to \$1.28/lb, and orange \$608/ton to \$1140/ton, in 2014-2017 (Index Mundi). Commodity price forecast has achieved high accuracy and been widely used for production decisions \cite{add_ref_2}.
  In this scenario, different treatments can be considered as arms. The effectiveness of a particular treatment is captured by the value of the arm, and is independent of the market price of the product. (The latter is true because an experiment in one farm, among tens of thousands of such farms in the US, has negligible impact on the overall production and thus the commodity's market price.)  The monetary reward is proportional to price and to the effectiveness of the treatment. The goal of a producer is to minimize the overall monetary regret, compared to the oracle. Therefore, intuitively, when the product price is low, the monetary regret of pulling a suboptimal arm is low, and vice versa.

\textbf{Motivating scenario 2: load variation.} 
Network configuration is widely used  in wireless networks, data-center networks, and the Internet, in order to control network topology, routing, load balancing, and thus improve the overall performance.
For example,
in a cellular network, a cell tower has a number of parameters to configure, including radio spectrum,  transmission power, antenna angle and direction, etc. The configuration of such parameters can greatly impact the overall performance, e.g., coverage, throughput, and service quality. A network  configuration can be considered as an  arm, where its performance needs to be learned.  Networks are typically designed and configured to handle the peak load, and thus we hope to learn the best configuration for the peak load. 

Network traffic load fluctuates over time. When the network load is low, we can inject dummy  traffic 
into the network so that the total load, the real load plus the dummy load, resembles the peak load. It allows us to learn the performance of the configuration under the peak load.  At the same time, the regret of using a suboptimal configuration is low because the real load affected  is low.  Furthermore, in practice, we can set the priority of the dummy traffic to be lower than that of the real traffic. Because networks handle high priority traffic first, low priority traffic results in little or no impact on the high priority traffic \cite{add_ref_7}. In this case,  the regret on the actual load is further reduced, or even negligible (when the suboptimal configuration is sufficient to handle the real load).

\textbf{Opportunistic bandits.} Motivated by these application scenarios, we study opportunistic bandits in this paper. Specifically, we define   {\em opportunistic bandit}   as a bandit problem with the following characteristics:  1) The best arm does not change over time. 2) The exploration cost (regret) of a suboptimal arm varies depending on a time-varying external condition that we refer to as \textbf{load} (which is the price in the first scenario). 3) The  load  is revealed before an arm is pulled, so that one can decide which arm to pull depending on the load. As its name suggests, in opportunistic bandits, one can leverage the opportunities of load variation to achieve a lower regret.
 In addition to the previous two examples,
opportunistic bandit algorithms can be  applied to other scenarios that  share the above characteristics. 


We note that opportunistic bandits  significantly differs from non-stationary bandits \cite{Garivier2011nonstationaryUCB,Besbes2014NIPS}. In non-stationary bandits, the expected reward of each arm varies and the optimal arm may change over time, e.g., because of the shift of interests. In opportunistic bandits, the optimal arm does not change over time, but the regret of trying a suboptimal arm changes depending on the  load. In other words, in non-stationary bandits, the dynamics of the optimal arm  make finding the optimal arm more challenging. In contrast, in opportunistic bandits, the time-varying nature of the load provides opportunities to reduce the regret of finding the fixed optimal arm. Because of such fundamental differences,
in non-stationary bandits, one can show polynomial regret (e.g., $\Omega(T^{2/3})$ \cite{Besbes2014NIPS}) because one has to keep track of the optimal arm. In opportunistic bandits, we can show $O(\log T)$ (or even  $O(1)$ in certain special cases) regret because we can push more exploration to slots when the regret is lower.

We also note the connection and difference between opportunistic bandits and contextual bandits \cite{Zhou2015CMAB:Survey,Wu2015NIPS:CCB,Li2010WWW:LinUCB,Chu2011AISTAT}. Broadly speaking, opportunistic bandits can be considered as a special case of contextual bandits where we can consider the load as the context.  However, general contextual bandits do not take advantages of the unique properties of opportunistic bandits, in particular, the optimal bandit remains the same, and regrets differ under different contexts (i.e., load). To follow this line, the performance of contextual bandits has been compared in Appendix \ref{app:linucb}. 

\textbf{Contributions.} In this paper, we propose an Adaptive Upper-Confidence-Bound (AdaUCB) algorithm to dynamically balance the exploration-exploitation tradeoff in opportunistic bandits. The intuition is clear: we should explore more when the load is low and exploit more when the load is high. The design challenge is to quantify the right amount of exploration and exploitation depending on the load. The analysis challenge is due to the inherent coupling over time and thus over bandits under different  conditions. 
In particular, due to the randomness nature of bandits, the empirical estimates  of the expected rewards could deviate from the true values, which could lead to suboptimal actions  when the load is high.
We address these challenges by studying the lower bounds on the number of pulls of the suboptimal arms under low load. 
Because the exploration factor is smaller under high load than that under low load, it requires less information accuracy to make the optimal decision  under high load. Thus, with an appropriate lower bound on  the number of pulls of the suboptimal arms under low load, we can show that
  the information obtained from the exploration under the low load is sufficient for accurate decisions under the high load. As a result,   the exploration under high load is reduced and thus so does  the overall regret.
%

 To the best of our knowledge, this is \textbf{the first work proposing and studying opportunistic bandits} that aims to adaptively balance the exploration-exploitation tradeoff considering load-dependent regrets.
We propose AdaUCB, an algorithm that adjusts the exploration-exploitation tradeoff according to the load level. 
 We prove that AdaUCB achieves $O(\log T)$ regret with a smaller coefficient than the traditional UCB algorithm. Furthermore, AdaUCB achieves $O(1)$ regret with respect to $T$ in the case where the exploration cost is zero when the load level is smaller than a certain threshold. 
Using both synthetic and real-world traces, we show that AdaUCB significantly outperforms other bandit algorithms, such as UCB and TS (Thompson Sampling), under large load fluctuations.


\section{System Model}
\label{sec:model}


We study an opportunistic bandit problem, where the exploration cost varies over time depending on an external condition, called \textbf{load} here. Specifically, consider a $K$-armed stochastic bandit system.
At time $t$, each arm has a random {\it nominal  reward} $X_{k,t}$, where $X_{k,t} \in [0,1]$ are independent across arms, and i.i.d.~over time, with mean value $\mathbb{E}[X_{k,t}] = u_k$. Let $u^* = \max_k u_k$ be the maximum expected reward and $k^* = \argmax u_k$ be the best arm. The arm with the best nominal reward does not depend on the load and does not change over time. 

Let  $L_t \geq 0$  be the load at time $t$. For simplicity, we assume $L_t \in [0,1]$.  The agent observes the value of  $L_t $ before making the decision; i.e.,  the agent pulls an arm $a_t$ based on both $L_t$ and the historical observations, i.e.,
$a_t = \Gamma(L_t, \mathcal{H}_{t-1})$,
where $\mathcal{H}_{t-1} = (L_1, a_1, X_{a_1, 1},  \ldots,  L_{t-1}, a_{t-1}, X_{a_{t-1}, {t-1}})$ represents the historical observations. The agent then receives an {\it actual reward} $L_t X_{a_t, t}$. While the underlying nominal reward $X_{a_t,t}$ is independent of $L_t$ conditioned on $a_t$, the actual reward depends on $L_t$. We also assume that the agent can observe the value of $X_{a_t,t}$  after pulling arm $a_t$ at time $t$.

 This model captures the essence of opportunistic bandits and its assumptions are reasonable.
 For example, 
in the agriculture scenario,  $X_{a_t, t}$ captures the effectiveness of a treatment, e.g., the survival rate or  the  yield of an antibiotic treatment. The value of $X_{a_t, t}$   can always be observed by the agent after applying treatment $a_t$ at time $t$. Conditioned on $a_t$,  $X_{a_t, t}$ is also  independent of $L_t$,  the price of the commodity.  Meanwhile, the actual reward, i.e., the  monetary reward,  is modulated by $L_t$ (the price)  as $L_t X_{a_t, t}$. 
 In the network configuration example, $X_{a_t, t}$  captures the impact of a configuration at the peak load, e.g.,  success rate, throughput, or service quality score.  Because the total load (the real load plus the dummy load) resembles the peak load, $X_{a_t,t}$ is independent of the real load $L_t$ conditioned on $a_t$, and can always be observed. Further, because the real load is a portion of the total load and the network can identify real traffic from dummy traffic, the actual reward is thus a portion of the total reward, modulated by the real load as  $L_t X_{a_t, t}$.

If system statistics are known {\it a priori}, then the agent will always pull the best arm and obtain the expected total reward $u^* \mathbb{E}[\sum_{t=1}^T L_t]$. Thus, the regret of a policy $\Gamma$ is defined as
\begin{equation}
R_{\Gamma}(T) = u^* \mathbb{E}\big[\sum_{t=1}^T L_t\big] - \sum_{t=1}^T \mathbb{E}[L_t X_{a_t,t}].
\label{eq:regret}
\end{equation}
In particular, when $L_t$ is i.i.d.~over time with mean value $\mathbb{E}[L_t] = \bar{L}$, the total expected reward for the oracle solution is $u^* \bar{L}T$ and the regret is  $R_{\Gamma}(T) = u^* \bar{L}T - \sum_{t=1}^T \mathbb{E}[L_t X_{a_t,t}]$.
Because the action $a_t$ can depend on $L_t$,  it is likely  that $\mathbb{E}[L_t X_{a_t,t}]\neq \bar{L}\mathbb{E}[X_{a_t,t}]$.



\section{Adaptive UCB} \label{sec:algorithm}


We first recall a general version of the classic UCB1 \cite{Auer2002ML:UCB} algorithm, referred to as UCB$(\alpha)$, which always selects the arm with the largest index defined in the following format:
\begin{equation}
\hat{u}_k(t) = \bar{u}_k(t) + \sqrt{\frac{\alpha \log t}{C_k(t-1)}}, ~1 \leq k \leq K, \nonumber \\
\end{equation}
where $\alpha$ is a constant, $C_k(t-1)$ is the number of pulls for arm-$k$ before $t$, and $\bar{u}_k(t) = \frac{1}{C_k(t-1)} \sum_{\tau=1}^{t-1}\mathds{1}(a_{\tau} = k) X_{k,\tau}$. It has been shown that UCB$(\alpha)$ achieves logarithmic regret in stochastic bandits when $\alpha > 1/2$ \cite{Bubeck2010PhD:bandits}. UCB1 in \cite{Auer2002ML:UCB} is a special case with $\alpha = 2$.

\begin{algorithm}[H]
    \caption{AdaUCB}
      \label{alg:adaUCB}
      \renewcommand{\algorithmicrequire}{\textbf{Input:}}
      \renewcommand{\algorithmicensure}{\textbf{Output:}}
      \renewcommand\algorithmiccomment[1]{%
      {//{\it ~{#1}}}%
      }
   \begin{algorithmic}[1]
        \STATE{\bfseries Init:} {$\alpha > 0.5$, $C_k(t) = 0$, $\bar{u}_{k}(t) = 1$.}
        \FOR{$t = 1$ {\bfseries to} $K$}
        \STATE{Pull each arm once and update $C_k(t)$ and $\bar{u}_k(t)$ accordingly;}
        \ENDFOR
        \FOR{$t = K+1$ {\bfseries to} $T$}
        \STATE{Observe $L_t$;}
        \STATE{Calculate UCB: for $k = 1,2,\ldots, K$,
        \begin{equation}
           \hat{u}_k(t) =
           \bar{u}_k(t) + \sqrt{\frac{\alpha(1 - \tilde{L}_t)\log t}{C_k(t-1)}},
           \end{equation}
          where $\tilde{L}_t$ is the normalized load defined in Eq.~\eqref{eq:norm_Lt};}
          \STATE{Pull the arm with the largest $\hat{u}_k(t)$:
        \begin{equation}
            a_t = \argmax_{1 \leq k \leq K}\hat{u}_k(t);
        \end{equation}
        }
        \STATE{Update $\bar{u}_k(t)$ and $C_k(t)$;}
       \ENDFOR
     \end{algorithmic}
 \end{algorithm} 

In this work, we propose an AdaUCB algorithm for opportunistic bandits.
In order to capture different ranges of $L_t$, we first normalize $L_t$ to be within $[0,1]$:
\begin{equation}\label{eq:norm_Lt}
\tilde{L}_t = \frac{[L_t]_{l^{(-)}}^{l^{(+)}} - l^{(-)}}{l^{(+)} - l^{(-)}},
\end{equation}
where $l^{(-)}$ and $l^{(+)}$ are the lower and upper thresholds for truncating the load level, and $[L_t]_{l^{(-)}}^{l^{(+)}} = \max\{l^{(-)}, \min(L_t, l^{(+)})\}$. 
Load normalization reduces the impact of different load distributions. It also restricts the coefficient of the exploration term in the UCB indices, which avoids under or over explorations.
To achieve good performance, the truncation thresholds should be appropriately chosen and can be learned online in practice, as discussed in Sec.~\ref{subsec:continuous_load}. We note that $\tilde{L}_t $ is only used in AdaUCB algorithm. The rewards and regrets are based on $L_t$, not $\tilde{L}_t $.

The AdaUCB algorithm adjusts the tradeoff between exploration and exploitation based on the load level $L_t$. Specifically, as shown in Algorithm~\ref{alg:adaUCB}, AdaUCB makes decisions based on the sum of the empirical reward (the exploitation term) $\bar{u}_k(t)$ and the confidence interval width (the exploration term). The latter term is  proportional to $\sqrt{1 - \tilde{L}_t}$. In other words, AdaUCB uses an exploration factor $\alpha(1-\tilde{L}_t)$ that is linearly decreasing in $\tilde{L}_t$. Thus, when the load level is high, the exploration term is relatively small and AdaUCB tends to emphasize exploitation, i.e., choosing the arms that perform well in the past.
In contrast, when the load level is low, AdaUCB uses a larger exploration term and  gives more opportunities to the arms with less explorations. Intuitively, with this load-awareness, AdaUCB explores more when the load is low and leverages the learned statistics to make better decisions when the load is high. Since the actual regret is scaled with the load level, AdaUCB can achieve an overall lower regret.
Note that we have experimented a variety of load adaptation functions. The current  one achieves superior empirical performance and is amenable to analyze, and thus adopted here. 

\section{Regret Analysis} \label{sec:regret_analysis}

Although the intuition behind AdaUCB is natural, the rigorous analysis of its regret is challenging. To analyze the decision in each slot, we require the statistics for the number of pulls of each arm. Unlike traditional regret analysis, we  care about not only the upper bound, but also the lower bound for calculating the confidence level. However, even for fixed load levels, it is difficult to characterize the total number of pulls for suboptimal arms, i.e., obtaining tight lower and upper bounds for the regret. The gap between the lower and upper bounds makes it more difficult to evaluate the properties  of UCB for general random load levels.
To make the intuition more clear and analyses more readable, we start with the case of squared periodic wave load and Dirac rewards to illustrate the behavior of AdaUCB in Sec.~\ref{subsec:Dirac}. Then, we extend the results to the case with random binary-value load and random rewards in Sec.~\ref{subsec:random}, and finally analyze the case with continuous load in Sec.~\ref{subsec:continuous_load}.

Specifically, we first consider the case with binary-valued load, i.e., $L_t \in \{\epsilon_0, 1-\epsilon_1\}$,
where $\epsilon_0, \epsilon_1 \in [0, 0.5)$. For this case, we let $l^{(-)} = \epsilon_0$ and $l^{(+)} = 1$. Then, $\tilde{L}_t = 0$ if $L_t = \epsilon_0$, and $\tilde{L}_t = \frac{1-\epsilon_0 - \epsilon_1}{1-\epsilon_0} = 1 - \frac{\epsilon_1}{1-\epsilon_0}$ if $L_t = 1 - \epsilon_1$. Therefore, the indices used by AdaUCB are given as follows:
\begin{equation}
\hspace{-0.1cm}\hat{u}_k(t) =
\begin{cases}
\bar{u}_k(t) + \sqrt{\frac{\alpha \log t}{C_k(t-1)}}, &\text{if $L_t = \epsilon_0$},\\
\bar{u}_k(t) + \sqrt{\frac{\alpha\epsilon_1 \log t}{(1-\epsilon_0)C_k(t-1)}}, &\text{if $L_t = 1-\epsilon_1$}.
\end{cases}
\end{equation}

We investigate the regret of AdaUCB under the binary-valued load described above in Sec.~\ref{subsec:Dirac} and Sec.~\ref{subsec:random}, and then study its performance under  continuous load in Sec.~\ref{subsec:continuous_load} with the insights obtained from the binary-valued load case.


\subsection{AdaUCB under Periodic Square Wave Load and Dirac Rewards} \label{subsec:Dirac}

We first study a simple case with periodic square wave load and Dirac rewards. In this scenario,  the evolution of the system under AdaUCB is deterministic.
The analysis of this deterministic system allows us to better understand AdaUCB and  quantify the benefit of  load-awareness. In addition, we focus on 2-armed bandits in analysis for easy illustration in this section. 

Specifically, we assume the load is
$L_t = \epsilon_0$ if $t$ is even, and $1 - \epsilon_1$ if $t$ is odd.
Moreover, the rewards are fixed, i.e.,  $X_{k,t} = u_k$ for all $k$ and $t$,  but unknown {\it a priori}. Without loss of generality, we assume arm-1 has higher reward, i.e., $1 \geq u_1 > u_2 \geq 0$, and let $\Delta = u_1 - u_2$ be the reward difference.

Under these settings,  we can obtain the bounds for the number of pulls for each arm by borrowing the idea from \cite{Salomon2011TR,Salomon2013JMLR}. The proofs of these results are included in Appendix~\ref{app:analysis_dirac}, which are  similar to \cite{Salomon2011TR,Salomon2013JMLR}, except for the effort of addressing the case of $L_t = 1-\epsilon_1$.

We first characterize the upper and lower bounds on the total number of pulls for the suboptimal arm.
\begin{lemma} \label{thm:adaUCB_dirac_bounds}
In the opportunistic bandit with periodic square wave load and Dirac rewards, the number of pulls for arm-2 under AdaUCB is bounded  as follows:\\
1) Upper bound for any $t \geq 1$:
$
C_2(t) \leq \frac{\alpha log t}{\Delta^2} + 1
$;\\
2) Lower bound for any $t = 2\tau \geq 2$: \\
$
C_2(2\tau) \geq f(\tau) = \int_2^\tau \min(h'(s),1){\rm d}s - h(2)
$,
where $h(s) = \frac{\alpha \log s}{\Delta^2}\big(1+\sqrt{\frac{2\alpha\log s}{(2s-1)\Delta^2}}\big)^{-2}$.
\end{lemma}

Note that $C_2(2\tau)$ provides the information for making decision in slot $2\tau + 1$, when $L_{t} = 1-\epsilon_1$.  With the lower bound in Lemma~\ref{thm:adaUCB_dirac_bounds}, we can show that after a certain time, AdaUCB will always pull the better arm when $L_t = 1- \epsilon_1$ with the information provided by $C_2(2\tau)$. Combining with the upper bound on $C_2(t)$, we can obtain the regret bound for AdaUCB:
\begin{theorem}\label{thm:adaUCB_dirac_regret}
In the opportunistic bandit with periodic square wave load and Dirac rewards, the regret of AdaUCB is bounded as:
$
R_{\text{AdaUCB}}(T) \leq \frac{\epsilon_0 \alpha\log T}{\Delta} + O(1)
$.
\end{theorem}

\textit{Remark 1:} According to \cite{Salomon2011TR}, the regret of UCB($\alpha$) is lower bounded by $\frac{\alpha \log T}{\Delta}$ for fixed load $L_t = 1$. Without load-awareness, we can expect that the explorations occur roughly uniformly under different load levels. Thus, the regret of UCB($\alpha$) in this opportunistic bandit is roughly  $\frac{\alpha(1+\epsilon_0 - \epsilon_1) \log T}{2\Delta}$, and is much larger than the regret of AdaUCB for small $\epsilon_0$ and $\epsilon_1$. As an extreme case, when $\epsilon_0 = 0$, the regret of AdaUCB is $O(1)$, while that of  UCB$(\alpha)$ is $O(\log T)$.

\textit{Remark 2:} The above analysis provides us insights about the benefit of load-awareness in opportunistic bandits. With load-awareness, AdaUCB forces exploration to the slots with lower load and the information obtained there is sufficient to make good decisions in higher-load slots. Thus, the overall regret of AdaUCB is much smaller than traditional load-agnostic algorithms.

\subsection{AdaUCB under Random Binary-Valued Load and Random Rewards} \label{subsec:random}
We now consider the more general case with random binary-valued load and random rewards.  We assume that  load $L_t \in \{\epsilon_0, 1 - \epsilon_1\}$ and $\mathbb{P}\{L_t = \epsilon_0\} = \rho \in (0, 1)$.
We consider i.i.d~random reward $X_{k, t} \in [0, 1]$ and $\mathbb{E}[X_{k,t} ] = u_k$, where $1 \geq u_1 > u_2 \geq u_3 \geq ... \geq u_K \geq 0$. Let $\Delta_k = u_1 - u_k$, and $\Delta^* = \min_{k > 1} \Delta_k = \Delta_2$ be the minimum gap between the suboptimal arms and the optimal arm. 

Compared with the deterministic case in Sec.~\ref{subsec:Dirac}, the analysis under random load and rewards is much more challenging. In particular, due to the reward randomness, the empirical value $\bar{u}_k(t)$ will deviate from its true value $u_k$. Unlike Dirac reward, this deviation could result in suboptimal decisions even when $\epsilon_0$ and $\epsilon_1$ are small. Thus, we need to carefully lower bound the number of pulls for each arm so that the deviation is bounded with high probability. We only provide sketches for the proofs here due to the space limit and refer readers to Appendix~\ref{app:analysis_randomload} for more detailed analyses.

We consider a larger $\alpha$ ($\alpha > 2$ in general, or  larger when explicitly stated) for theoretical analysis purpose, similarly to earlier UCB papers such as \cite{Auer2002ML:UCB}. As we will see in the simulations, AdaUCB with $\alpha > 1/2$ works well under general random load.

We first propose a loose but useful bound for the number of pulls for the optimal arm. Let $C_k^{(0)}(t)$ be the number of slots where arm-$k$ is pulled when $L_t = \epsilon_0$, i.e., $C_k^{(0)}(t) = \sum_{\tau = 1}^t\mathds{1}(L_{\tau} = \epsilon_0,a_{\tau} = k)$.
 \begin{lemma} \label{thm:lower_bound_opt_arm_randomload}
In the opportunistic bandit with random binary-valued load and random rewards,  for a constant  $\eta \in (0, \rho)$, there exists a constant $T_2$, such that under AdaUCB, for all $t \geq T_2$
\begin{eqnarray}
&& \mathbb{P}\big\{C_1^{(0)}(t) < \frac{(\rho-\eta) t}{2}\big\}  \nonumber  \\
&\leq& e^{-2\eta^2 t} + \frac{[2(K-1)]^{2\alpha - 1}}{2\alpha - 2} \big[(\rho - \eta)t \big]^{-2\alpha + 2}.\nonumber
\end{eqnarray}
\end{lemma}

\textit {Sketch of Proof:} 
The key intuition of proof is that when $C_1^{(0)}(t)$ is too small, the optimal arm will be pulled with high probability. Specifically, let $k' > 1$ be the index of arm that has been pulled for the most time among the suboptimal arms before $t$, and  $t' < t$ be the last slot when $k'$ is pulled under load $L_t = \epsilon_0$ for the last time. If $C_1^{(0)}(t) <  \frac{(\rho-\eta)  t}{2}$, then $C_{k'}(t'-1) \geq C_{k'}^{(0)}(t'-1) = \Theta(t)$ with high probability. Using the fact that $\frac{\log t}{t} \to 0$ as $t \to \infty$, we know there exists a constant $T_2$ such that for $t \geq T_2$, the confidence width $\sqrt{\frac{\alpha\log t'}{C_{k'}(t'-1)}}$ will be sufficiently small compared with the minimum gap $\Delta^* \leq \Delta_k$. Moreover, the algorithm will pull the best arm when the UCB deviation is sufficiently small. Then, we can bound the probability of the event $C_1^{(0)}(t) < \frac{(\rho-\eta) t}{2}$ by bounding the deviation of UCBs. 

Next we bound the total number of pulls of the suboptimal arm as follows. 
\begin{lemma} \label{thm:random_reward_general_upperbound_randomload}
In the opportunistic bandit with random binary-valued load and random rewards, under AdaUCB, we have
\begin{equation}
\mathbb{E}[C_k(T)] \leq \frac{4\alpha\log T}{\Delta_k^2} + O(1), 1 < k \leq K.
\end{equation}
\end{lemma}

\textit{Sketch of Proof:} To prove this lemma, we discuss the slots when the suboptimal arm is pulled under low and high load levels, respectively. When the load is low, i.e., $L_t = \epsilon_0$, AdaUCB becomes UCB$(\alpha)$ and thus we can bound the probability of pulling the suboptimal arm similarly to \cite{Auer2002ML:UCB}. When the load is high, i.e., $L_t = 1-\epsilon_1$, the index becomes $\hat{u}_k(t) = \bar{u}_k(t) + \sqrt{\frac{\alpha \epsilon_1 \log t}{(1-\epsilon_0)C_k(t-1)}}$. In this case, with high probability, the index of the optimal arm is lower bounded by $u_1 - \left(1-\sqrt{\frac{\epsilon_1}{(1-\epsilon_0)}}\right)\sqrt{\frac{\alpha \log t}{C_1(t-1)}}$ according to Lemma~\ref{thm:lower_bound_opt_arm_randomload}. With similar adjustment on the UCB index for the suboptimal arm, we can bound the probability of pulling the suboptimal arm under high load. The conclusion of the lemma then follows by combining the above two cases.

Now we further lower bound the pulls of the suboptimal arm with high probability.
\begin{lemma} \label{thm:random_reward_lowerbound_randomload}
In the opportunistic bandit with random binary-valued load and random rewards, for a positive number $\delta \in (0,1)$, we have for any $k > 1$,
\begin{eqnarray}
&& \mathbb{P}\bigg\{C_k(t) < \frac{\alpha \log t}{4(\Delta_k+\delta)^2} \bigg\}  \nonumber \\
&=& O\big(t^{-(2\alpha -3)} +  t^{-(2\alpha(\frac{1-\delta}{2-\delta})^2 - 2)}\big). \nonumber
\end{eqnarray}

\end{lemma}
\textit {Sketch of Proof:} Although the analysis is more difficult, the intuition of proving this lemma is similar to that of Lemma~\ref{thm:lower_bound_opt_arm_randomload}: if $C_k(t)$ is too small at a certain slot, then we will pull the suboptimal arm instead of the optimal arm with high probability. To be more specific, we focus on the slot $t'$ when the optimal arm is pulled for the last time before $t$ under load $L_t = \epsilon_0$. According to Lemma~\ref{thm:lower_bound_opt_arm_randomload}, $C_1(t) \geq C_1^{(0)}(t) \geq \frac{(\rho-\eta) t}{2}$ with high probability, indicating $t' \geq (\rho-\eta) t/2$ with high probability. Moreover, the index for the optimal arm $\hat{u}_1(t') \leq u_1 +\delta$ with high probability for a sufficiently large $t'$, because $\sqrt{\frac{\log t}{t}} \to 0$ as $t \to \infty$. On the other hand, we can show that for the suboptimal arm, $\hat{u}_k(t') > u_1 +\delta = u_k + (\Delta_k + \delta)$ with high probability when $C_k(t'-1) < \frac{\alpha \log t}{4(\Delta_k+\delta)^2}$. Thus, the probability of pulling the optimal arm at $t'$ is bounded by a small value, implying the conclusion of the lemma.

Using the above lemmas,  now we can further refine the upper bound on the regret of AdaUCB and show that AdaUCB achieves smaller regret than traditional UCB.
\begin{theorem}\label{thm:adaUCB_random_regret_randomload}
Using AdaUCB in the opportunistic bandit with random binary-valued load and random rewards, if $\alpha > 16$ and $\sqrt{\frac{\epsilon_1}{1-\epsilon_0}} < \frac{1}{8}$, we have
\begin{equation}
R_{\text{AdaUCB}}(T) \leq 4 \epsilon_0 \alpha\log T \sum_{k > 1} \frac{1}{\Delta_k} + O(1).
\end{equation}
\end{theorem}
\textit {Sketch of Proof:} The key idea of the proof is to find an appropriate $\delta \in (0, \Delta^*)$, such that $\alpha > 16(1+\frac{\delta}{\Delta^*})^2$ and $\sqrt{\frac{\epsilon_1}{1-\epsilon_0}} < \frac{\Delta^*}{8(\Delta ^*+ \delta)}$. In fact, the existence of this $\delta$ is guaranteed under the assumptions $\alpha > 16$ and $\sqrt{\frac{\epsilon_1}{1-\epsilon_0}} < \frac{1}{8}$. Using this $\delta$,  we can then use  Lemma~\ref{thm:random_reward_lowerbound_randomload} to bound the probability of pulling the suboptimal arm when the load is high. This indicates most explorations occur when the load is low, i.e., $L_t = \epsilon_0$. The conclusion of this theorem then follows according to Lemma~\ref{thm:random_reward_general_upperbound_randomload}.

\textit {Remark 3:} Although there is no tight lower bound for the regret of UCB$(\alpha)$, we know that for traditional (load-oblivious) bandit algorithms, $\mathbb{E}[C_k(T)]$ is lower bounded by $\frac{\log T}{KL(u_k, u_1)}$ \cite{Lai1985AAM} for large $T$, where $KL(u_k, u_1)$ is the Kullback-Leibler divergence. Without load-awareness, the regret will be roughly lower bounded by $\frac{(1-\epsilon_0-\epsilon_1)\log T}{2} \sum_{k > 1} \frac{\Delta_k }{KL(u_k, u_1)}$. In contrast, with load-awareness, AdaUCB can achieve much lower regret than load-oblivious algorithms, when the load fluctuation is large, i.e., $\epsilon_0$ and $\epsilon_1$ are small.

Theorem \ref{thm:adaUCB_random_regret_randomload} directly implies the following result.
\begin{corol}\label{thm:adaUCB_O1_regret_randomload}
Using AdaUCB in the opportunistic bandit with random reward under i.i.d.~random binary load where $\epsilon_0=0$, if $\alpha > 16$ and $\epsilon_1 < \frac{\sqrt{2}}{4}$, we have
$
R_{\text{AdaUCB}}(T) =  O(1)
$.
\end{corol}

\textit {Remark 4:} We note that this $O(1)$ bound is in the sense of expected regret, which is different from  the high probability $O(1)$ regret bound \cite{Abbasi2011NIPS}.
Specifically, while the opportunistic bandits can model the whole spectrum of load-dependent regret, Corollary~\ref{thm:adaUCB_O1_regret_randomload} highlights one end of the spectrum where there are ``free'' learning opportunities. In this case, we push most explorations to the ``free'' exploration  slots and result in an $O(1)$ expected regret. Note that even under ``free'' exploration, we assume here that the value of the arms can be observed as discussed in Sec.~\ref{sec:model}.

It is  worth noting that there are realistic scenarios where the exploration cost of a suboptimal arm is zero or close to zero.  Consider the network configuration case where we use  throughput as the reward. In this case,  $X_{a_t, t}$ is the percentage of the peak load that configuration $a_t$ can handle.
Because of the dummy low-priority traffic injected into the network, we can  learn the true value of $X_{a_t,t}$ under the peak load.
At the same time, configuration  $a_t$, although suboptimal,  may  completely satisfy the real load  $L_t$ because it is high priority and thus served first. Therefore, although a suboptimal arm, $a_t$ sacrifices no throughput on the real load $L_t$, and thus generates a real regret of zero.  In other words, even if the system load is always positive, the  chance of zero regret under a suboptimal arm is greater than zero, and in practice, can be non-negligible. To capture this effect, we can modify the regret defined in Eq.~\eqref{eq:regret} by replacing $L_t$ with 0 when $L_t$ is smaller than a threshold.

Last, we note that,  under the condition of Corollary~\ref{thm:adaUCB_O1_regret_randomload},  it is easy to design other heuristic algorithms that can perform well. For example, one can do round-robin exploration when the load is zero and chooses the best arm when the load is non-zero. However, such naive strategies are difficult to extend to more general cases. In contrast, AdaUCB applies to a wide range of situations, with both theoretical performance guarantees and desirable  empirical performance.


\textbf{Dependence on $\rho$:} In the regret analysis, we focus on the asymptotic behavior of the regret as $T$ goes to infinity. In the bound, the constant term contains the impact of other factors, in particular the ratio of low load $\rho$,  as shown in Appendix~\ref{subsec:impact_rho}. From the analysis, one can see that the constant term increases as $\rho \to 0$. It suggests that one should use the traditional UCB 
 when $\rho$ is small because  there exists little load fluctuation. In practice,  AdaUCB achieves much smaller regret than traditional UCB and TS algorithms, even for small values of $\rho$ such as $\rho = 0.05$ under binary load and $\rho = 0.001$ under continuous load. Such analysis and evaluations establish guidelines on when to use UCB or AdaUCB. More discussions can be found in Appendix~\ref{app:add_simulations}.

\subsection{AdaUCB under Continuous Load} \label{subsec:continuous_load}

Inspired by the insights obtained from the binary-valued load case, we discuss AdaUCB in opportunistic bandits under continuous load in this section.

\textbf{Selection of truncation thresholds.}
When the load is continuous, we need to choose appropriate  $l^{(-)}$ and $l^{(+)}$ for AdaUCB. We first assume that the load distribution is {\it a priori} known, and discuss how to choose the thresholds under unknown load distribution later. The analysis under binary-valued load indicates that, the explorations mainly occur in low load slots. To guarantee sufficient explorations for a logarithmic regret, we propose to select the thresholds such that:

\begin{itemize}
\item The lower threshold $l^{(-)}$ satisfies $\mathbb{P}\{L_t \leq l^{(-)}\} = \rho > 0$;
\item The upper threshold $l^{(+)} \geq l^{(-)}$.
\end{itemize}
In the special case of $l^{(+)} = l^{(-)}$, we redefine the normalized load $\tilde{L}_t$ in \eqref{eq:norm_Lt} as $\tilde{L}_t = 0$ when $L_t \leq l^{(-)}$ and $\tilde{L}_t = 1$ when $L_t > l^{(-)}$.

\textbf{Regret analysis.}
Under continuous load, it is hard to obtain regret bound as that in Theorem~\ref{thm:adaUCB_random_regret_randomload} for general $l^{(-)}$ and $l^{(+)}$ chosen above. Instead, we first show logarithmic regret for general $l^{(-)}$ and $l^{(+)}$, and then illustrate the advantages of AdaUCB for the special case with $l^{(-)} = l^{(+)}$.

First, we show that AdaUCB with appropriate truncation thresholds achieves logarithmic regret as below. This lemma is similar to Lemma~\ref{thm:random_reward_general_upperbound_randomload}, and the detailed outline of proof can be found in Appendix~\ref{app:continuous load}.
\begin{lemma} \label{thm:regret_upperbound_continuousload}
In the opportunistic bandit with random continuous load and random rewards, under AdaUCB with $\mathbb{P}\{L_t \leq l^{(-)}\} = \rho > 0$ and $l^{(+)} \geq l^{(-)}$, we have
\begin{equation}
\mathbb{E}[C_k(T)] \leq \frac{4\alpha\log T}{\Delta_k^2} + O(1).
\end{equation}
\end{lemma}

Next, we illustrate the advantages of AdaUCB under continuous load by studying the regret bound for AdaUCB with special thresholds $l^{(+)} = l^{(-)}$.
 \begin{theorem} \label{thm:regret_upperbound_continuousload_singlethreshold}
In the opportunistic bandit with random continuous load and random rewards, under AdaUCB with $\mathbb{P}\{L_t \leq l^{(-)}\} = \rho > 0$ and $l^{(+)}  =  l^{(-)}$ , we have
\begin{eqnarray}
R_{\text{AdaUCB}}(T) \leq 4  \alpha\log T \mathbb{E}[L_t | L_t \leq l^{(-)}] \sum_{k > 1}\frac{1}{\Delta_k} + O(1),
\end{eqnarray}
where $\mathbb{E}[L_t | L_t \leq l^{(-)}]$ is the expectation of $L_t$ conditioned on $L_t \leq l^{(-)}$.
\end{theorem}
\textit {Sketch of Proof:} Recall that for this special case $l^{(+)} = l^{(-)}$, we let  $\tilde{L}_t = 0$ for $L_t \leq l^{(-)}$ and $\tilde{L}_t = 1$ for $L_t > l^{{(+)}}$. Then we can prove the theorem analogically to the proof of Theorem~\ref{thm:adaUCB_random_regret_randomload} for the binary-valued case. Specially, when $L_t \leq l^{(-)}$, we have $\tilde{L}_t = 0$ and it corresponds to the case of $L_t = \epsilon_0$ ($\tilde{L}_t = 0$) in the binary-valued load case. Similarly, the case of  $L_t > l^{{(+)}}$ ($\tilde{L}_t = 1$) corresponds to the case of $L_t = 1 - \epsilon_1$ under binary-valued load with $\epsilon_1 = 0$. Then, we can obtain results similar to Lemma~\ref{thm:random_reward_lowerbound_randomload} and thus show that the regret under load $L_t > l^{{(+)}}$ is $O(1)$. Furthermore, the number of pulls under load level $L_t \leq l^{(-)}$ is bounded according to Lemma~\ref{thm:regret_upperbound_continuousload}. The conclusion of the theorem then follows by using the fact that all load below $l^{(-)}$ are treated the same by AdaUCB, i.e., $\tilde{L}_t = 0$ for all $L_t  \leq l^{(-)}$.

\textit{Remark 5:}
We compare the regret of AdaUCB and conventional bandit algorithms by an example, where the load level  $L_t$ is uniformly distributed in $[0,1]$. In this simple example, the regret of AdaUCB with thresholds $l^{(+)} = l^{(-)}$ is bounded by
$
R_{\text{AdaUCB}}(T) \leq  4  \alpha\log T \sum_{k> 1}\frac{1}{\Delta_k} \cdot \frac{\rho}{2} + O(1)
$,
since $\mathbb{E}[L_t | L_t \leq l^{(-)}] = \rho / 2$ and $\mathbb{E}[L_t | L_t > l^{(-)}]  < 1$. However, for any load-oblivious bandit algorithm such as UCB($\alpha$) , the regret is lower bounded by
$
\log T \sum_{k > 1} \frac{\Delta_k } {KL(u_k, u_1)} \cdot \frac{1}{2} + O(1)
$.
Thus, AdaUCB achieves much smaller regret when $T$ is large and $\rho$ is relatively small.

\textit {Remark 6:}  From the above analysis, we can see that the selection of $l^{(+)}$ does not affect the order of the regret ($O(\log T)$). However, for a fixed $l^{(-)}$, we can further adjust $l^{(+)}$ to control the explorations for the load in the range of $(l^{(-)}, l^{(+)})$. Specifically, with a larger $l^{(+)}$, more explorations happen under the load between $l^{(-)}$ and $l^{(+)}$. These explorations  accelerate the learning speed but may increase the long term regret because we allow more explorations under load $l^{(-)} < L_t < l^{(+)}$. The behavior is opposite if we use a smaller $l^{(+)}$.
In addition, appropriately chosen thresholds also handle the case when the load has little or no fluctuation, i.e., $L_t\approx c$. For example, if we set   $l^{(-)}=c$ and $l^{(+)}=2c$,  AdaUCB degenerates to UCB($\alpha$).

\textbf{E-AdaUCB.} In practice, the load distribution may be unknown {\it a priori} and may  change over time. To address this issue,
we propose a variant, named Empirical-AdaUCB (E-AdaUCB), which adjusts the thresholds $l^{(-)}$ and $l^{(+)}$ based on the empirical load distribution. Specifically, the algorithm maintains the histogram for the load levels (or its moving average version for non-stationary cases), and then select $l^{(-)}$ and $l^{(+)}$ accordingly. For example, we can select $l^{(-)}$ and $l^{(+)}$ such that the empirical probability $\mathbb{\tilde{P}}\{L_t \leq l^{(-)}\} =  \mathbb{\tilde{P}}\{L_t \geq l^{(+)}\} = 0.05$. We can see that,  in most  simulations,  E-AdaUCB performs closely to AdaUCB with thresholds chosen offline.

\section{Experiments} \label{sec:sim_results}

\begin{figure*}[thbp]
\begin{center}
\begin{minipage}{0.66\textwidth}
\begin{center}
\subfigure[Binary-valued load]{\includegraphics[angle = 0,width = 0.47\linewidth]{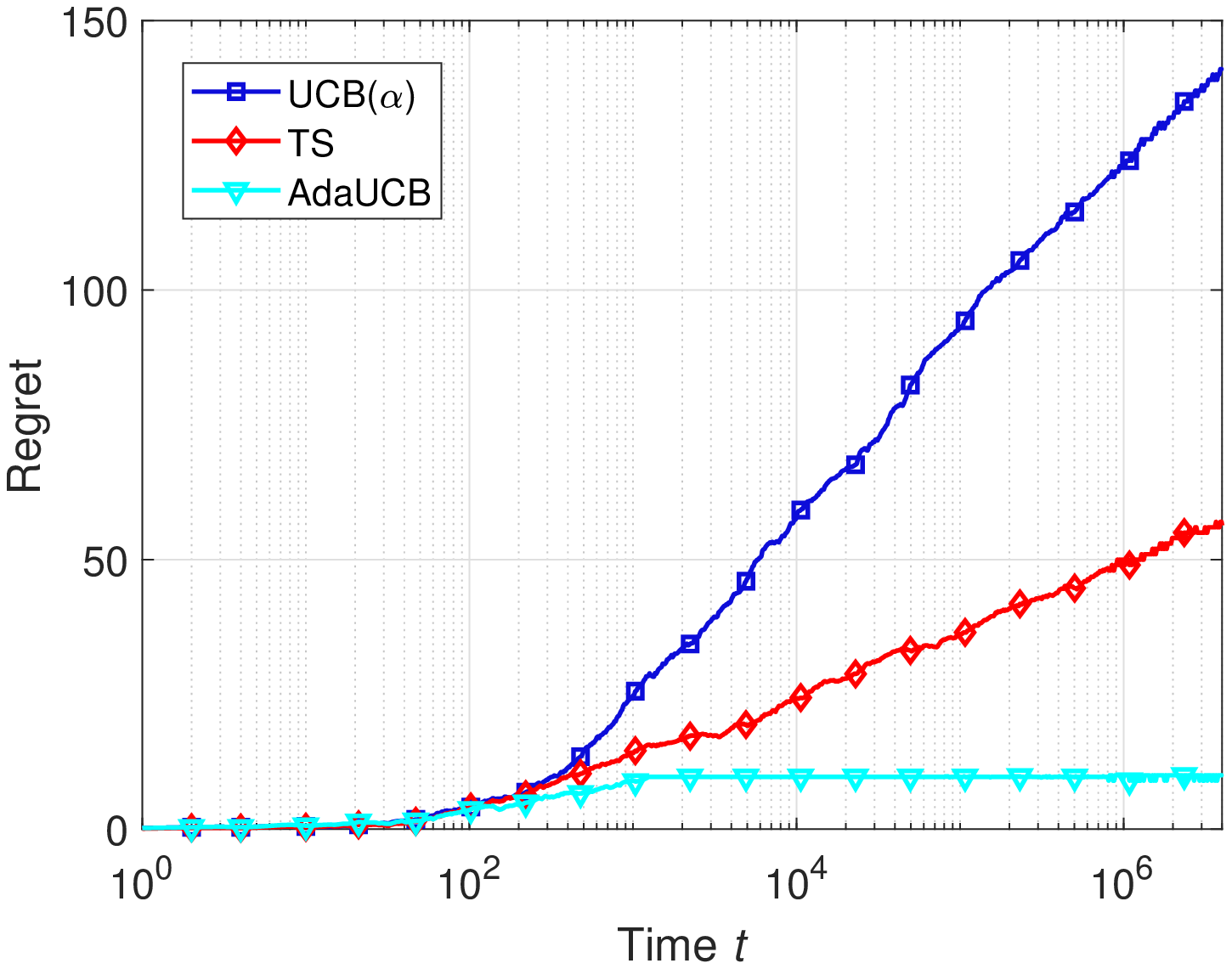}
\label{fig:BinaryValue_load}}
\subfigure[Beta distributed load]
{\includegraphics[angle = 0,width = 0.47\linewidth]{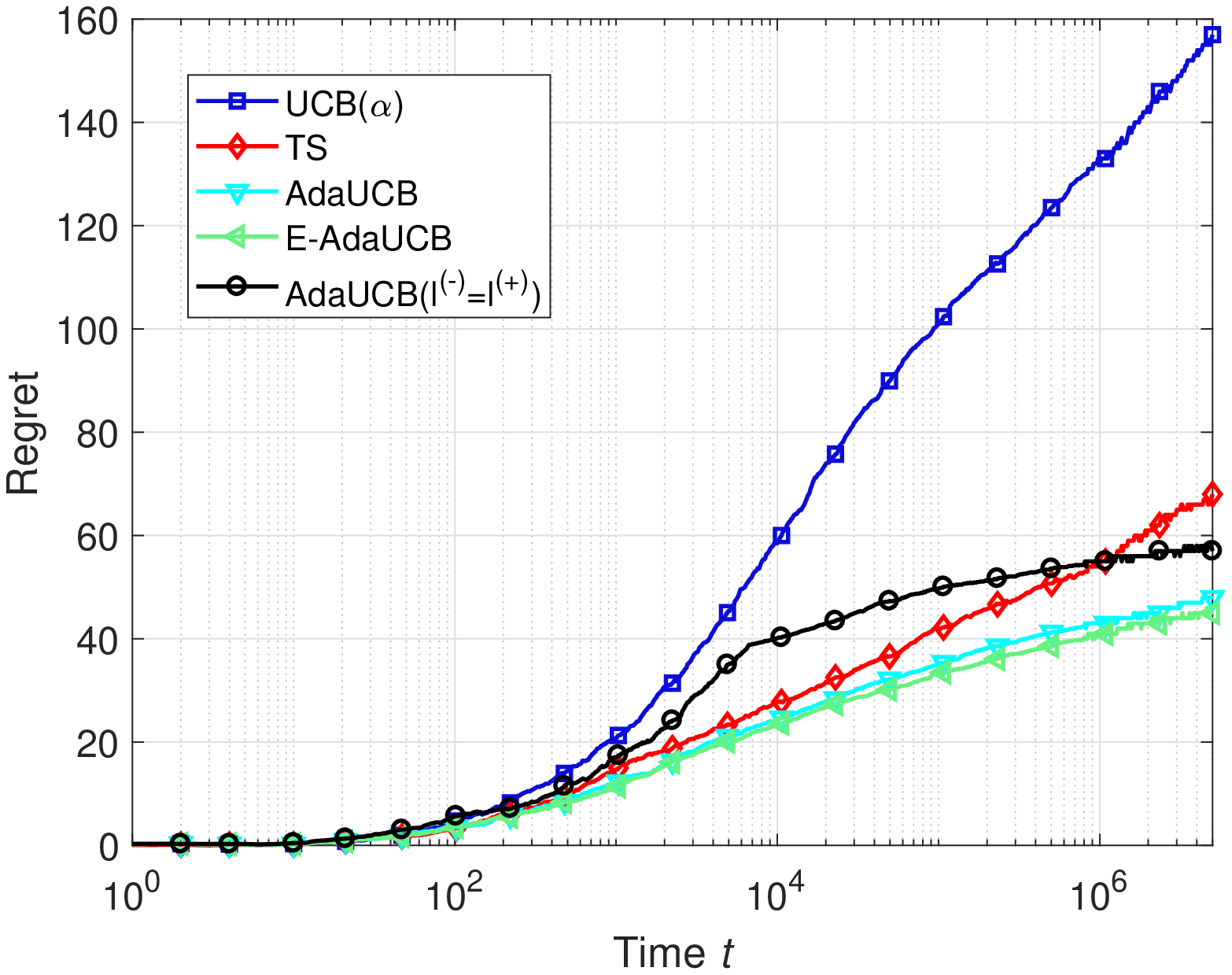}
\label{fig:beta_load}}
\vspace{-.25cm}
\caption{Regret under Synthetic Scenarios. In (a), $\epsilon_0 = \epsilon_1 = 0,  \rho = 0.5$. In (b), for AdaUCB, $l^{(-)} = l^{(-)}_{0.05}$,  $l^{(+)} = l^{(+)}_{0.05}$;
for AdaUCB($l^{(-)} =l^{(+)}$), $l^{(-)} =l^{(+)} = l^{(-)}_{0.05}$.}
\label{fig:synthetic_scenarios}
\end{center}
\end{minipage}
\begin{minipage}{0.33\textwidth}
\begin{center}
{\includegraphics[angle = 0,width = 0.94\linewidth]{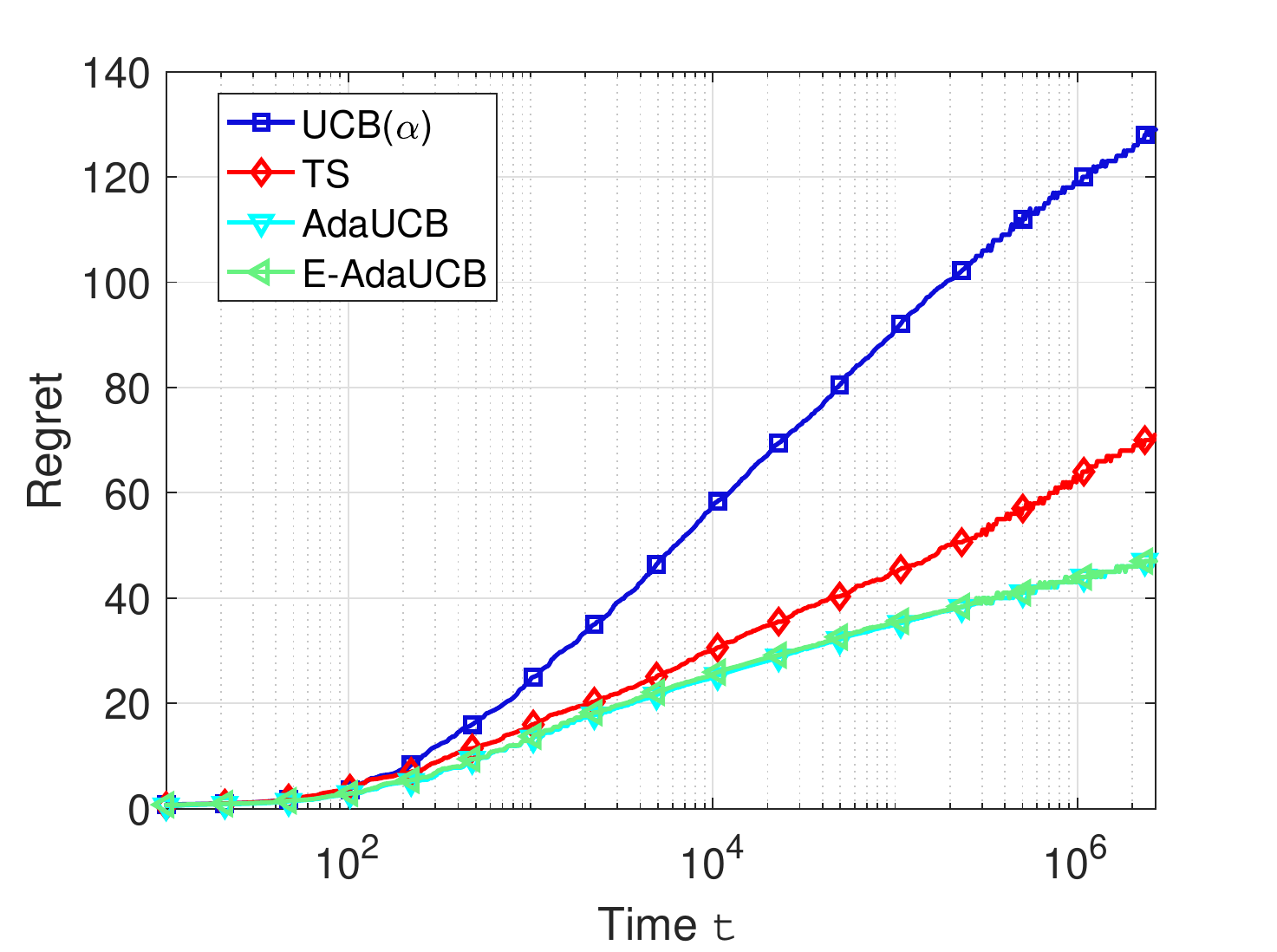}
\vspace{0.25cm}
\caption{Regret in MVNO systems.}
\vspace{0.25cm}
\label{fig:regret_mvno}}
\end{center}
\end{minipage}
\end{center}
\vspace{-0.45cm}
\end{figure*}

In this section, we evaluate the performance of AdaUCB using both synthetic data and real-world traces. We use the classic UCB$(\alpha)$ and TS (Thompson Sampling) algorithms as comparison baselines. In both AdaUCB and UCB$(\alpha)$, we set $\alpha$ as $\alpha = 0.51$, which is close to $1/2$ and performs better than a larger $\alpha$. We note that the gap between AdaUCB and the classic UCB$(\alpha)$ clearly demonstrates the impact of opportunistic learning.  On the other hand, TS is one of the most popular and robust bandit algorithms applied  to a wide range of application scenarios. So we apply it here as a reference. However, because AdaUCB and TS (or other bandit algorithms) improve UCB on different fronts, so their comparison does not clearly show the impact of opportunistic bandit. 

\textbf{AdaUCB under synthetic scenarios.}
We consider a 5-armed bandit with Bernoulli rewards, where the expected reward vector is $[0.05,0.1,0.15,0.2,0.25]$. 
Fig.~\ref{fig:BinaryValue_load} shows the regrets for different algorithms under random binary-value load with $\epsilon_0 = \epsilon_1 = 0$  and $\rho = 0.5$. AdaUCB significantly reduces the regret in opportunistic bandits. 
Specifically, the exploration cost in this case can be zero and AdaUCB achieves $O(1)$ regret. 
For continuous load, Fig.~\ref{fig:beta_load} shows the regrets for different algorithms with beta distributed load. AdaUCB still outperforms the UCB($\alpha$) or TS algorithms. Here, we define $l^{(-)}_{\rho}$ as the lower threshold such that $\mathbb{P}\{L_t \leq l^{(-)}_{\rho}\}=\rho$, and $l^{(+)}_{\rho}$ as the upper threshold such that $\mathbb{P}\{L_t \geq l^{(+)}_{\rho}\}=\rho$. 
These simulation results demonstrate that, with appropriately chosen parameters, the proposed AdaUCB and E-AdaUCB algorithms achieve good performance by leveraging the load fluctuation in opportunistic bandits. As a special case, with a single threshold $l^{(+)} =l^{(-)} = l^{(-)}_{0.05}$, AdaUCB still outperforms UCB($\alpha$) and TS, although it may have higher regret at the beginning.  
More simulation results can be found in Appendix~\ref{app:synthetic_simu}, where we study the impact of environment and algorithm parameters such as load fluctuation and the thresholds for load truncation. In particular, the results show that AdaUCB works well in continuous load when $\rho$ is very small. 

\textbf{AdaUCB applied in MVNO systems.}
We now evaluate the proposed algorithms using real-world traces. In an MVNO  (Mobile Virtual Network Operator) system, a virtual operator, such as Google Fi~\cite{googlefi}, provides services to users by leasing  network resources from real mobile operators. In such a system, the virtual operator would like to provide its users high quality service by accessing the network resources of the real operator with the best network performance. Therefore, we view each real mobile operator as an arm, and the quality of user experienced on that operator network as the reward.
We use  experiment data from  Speedometer~\cite{speedometer} and another anonymous operator to conduct the evaluation. More details about the MVNO system can be found in Appendix \ref{app:mvno}. 
Here, using insights obtained from simulations based on the synthetic data, we choose $l^{(-)}$ and $l^{(+)}$ such that $\mathbb{P}\{L_t \leq l^{(-)}\} = \mathbb{P}\{L_t \geq l^{(+)}\} = 0.05$.
As shown in Fig.~\ref{fig:regret_mvno}, the regret of AdaUCB is only about 1/3 of UCB$(\alpha)$, and the performance of E-AdaUCB is indistinguishable from that of AdaUCB. 
This experiment demonstrates the effectiveness of AdaUCB and E-AdaUCB in practical situations, where the load and the reward are continuous and are possibly non-stationary. It also demonstrates the practicality of E-AdaUCB without {\it a priori} load distribution information.


%
\section{Conclusions and Future Work} \label{sec:conclusion}
In this paper we study  opportunistic bandits  where the regret of pulling a suboptimal arm depends on  external conditions such as traffic load or produce price.
We propose   AdaUCB   that opportunistically chooses between exploration and exploitation based on the load level, i.e., taking the slots with low load level as opportunities for more explorations. We analyze the regret of AdaUCB, and show that AdaUCB can achieve provable lower regret than the traditional UCB algorithm, and even $O(1)$ regret with respect to time horizon $T$, under certain conditions. Experimental results based on both synthetic and real data  demonstrate the significant benefits of opportunistic exploration under large load fluctuations.

This work is a first attempt to study opportunistic bandits, and   several open questions remain. First, although AdaUCB achieves  promising experimental performance under  general settings, rigorous analysis with tighter performance bound remains challenging.  
Furthermore,  opportunistic TS-type  algorithms are also interesting because TS-type algorithms often performs better than UCB-type algorithms in practice. Last, we hope to investigate more general relations between the load and actual reward.


\section*{Acknowledgements}
This research was supported in part by NSF Grants  CCF-1423542, CNS-1547461, CNS-1718901. 
The authors would like to thank Prof.~Peter Auer (University of Leoben) for his helpful suggestions, 
and the reviewers for their valuable feedback.

\bibliographystyle{icml2018}
\bibliography{OnlineOpt,mypublications_0802,math}

\begin{thebibliography}{20}
\providecommand{\natexlab}[1]{#1}
\providecommand{\url}[1]{\texttt{#1}}
\expandafter\ifx\csname urlstyle\endcsname\relax
  \providecommand{\doi}[1]{doi: #1}\else
  \providecommand{\doi}{doi: \begingroup \urlstyle{rm}\Url}\fi

\bibitem[Abbasi-Yadkori et~al.(2011)Abbasi-Yadkori, P{\'a}l, and
  Szepesv{\'a}ri]{Abbasi2011NIPS}
Abbasi-Yadkori, Y., P{\'a}l, D., and Szepesv{\'a}ri, C.
\newblock Improved algorithms for linear stochastic bandits.
\newblock In \emph{Advances in Neural Information Processing Systems}, pp.\
  2312--2320, 2011.

\bibitem[Auer et~al.(2002)Auer, Cesa-Bianchi, and Fischer]{Auer2002ML:UCB}
Auer, P., Cesa-Bianchi, N., and Fischer, P.
\newblock Finite-time analysis of the multiarmed bandit problem.
\newblock \emph{Machine learning}, 47\penalty0 (2-3):\penalty0 235--256, 2002.

\bibitem[Besbes et~al.(2014)Besbes, Gur, and Zeevi]{Besbes2014NIPS}
Besbes, O., Gur, Y., and Zeevi, A.
\newblock Stochastic multi-armed-bandit problem with non-stationary rewards.
\newblock In \emph{Advances in neural information processing systems}, pp.\
  199--207, 2014.

\bibitem[Brandt \& Bessler(1983)Brandt and Bessler]{add_ref_2}
Brandt, J.~A. and Bessler, D.~A.
\newblock Price forecasting and evaluation: An application in agriculture.
\newblock \emph{Journal of Forecasting}, 2\penalty0 (3):\penalty0 237--248,
  1983.

\bibitem[Bubeck(2010)]{Bubeck2010PhD:bandits}
Bubeck, S.
\newblock \emph{Bandits games and clustering foundations}.
\newblock PhD thesis, Universit{\'e} des Sciences et Technologie de Lille-Lille
  I, 2010.

\bibitem[Bubeck \& Cesa-Bianchi(2012)Bubeck and Cesa-Bianchi]{Bubeck2012Survey}
Bubeck, S. and Cesa-Bianchi, N.
\newblock Regret analysis of stochastic and nonstochastic multi-armed bandit
  problems.
\newblock \emph{Machine Learning}, 5\penalty0 (1):\penalty0 1--122, 2012.

\bibitem[Chu et~al.(2011)Chu, Li, Reyzin, and Schapire]{Chu2011AISTAT}
Chu, W., Li, L., Reyzin, L., and Schapire, R.~E.
\newblock Contextual bandits with linear payoff functions.
\newblock In \emph{International Conference on Artificial Intelligence and
  Statistics}, pp.\  208--214, 2011.

\bibitem[Garivier \& Moulines(2011)Garivier and
  Moulines]{Garivier2011nonstationaryUCB}
Garivier, A. and Moulines, E.
\newblock On upper-confidence bound policies for non-stationary bandit
  problems.
\newblock In \emph{International Conference on Algorithmic Learning Theory},
  2011.

\bibitem[Lai(1987)]{Lai1987adaptive}
Lai, T.~L.
\newblock Adaptive treatment allocation and the multi-armed bandit problem.
\newblock \emph{The Annals of Statistics}, pp.\  1091--1114, 1987.

\bibitem[Lai \& Robbins(1985)Lai and Robbins]{Lai1985AAM}
Lai, T.~L. and Robbins, H.
\newblock Asymptotically efficient adaptive allocation rules.
\newblock \emph{Advances in Applied Mathematics}, 6\penalty0 (1):\penalty0
  4--22, 1985.

\bibitem[Li et~al.(2010)Li, Chu, Langford, and Schapire]{Li2010WWW:LinUCB}
Li, L., Chu, W., Langford, J., and Schapire, R.~E.
\newblock A contextual-bandit approach to personalized news article
  recommendation.
\newblock In \emph{ACM International Conference on World Wide Web (WWW)}, pp.\
  661--670, 2010.

\bibitem[Press(2009)]{Press2009NACt}
Press, W.~H.
\newblock Bandit solutions provide unified ethical models for randomized
  clinical trials and comparative effectiveness research.
\newblock \emph{Proceedings of the National Academy of Sciences}, 106\penalty0
  (52):\penalty0 22387--22392, 2009.

\bibitem[{Project Fi, {\it https://fi.google.com}}()]{googlefi}
{Project Fi, {\it https://fi.google.com}}.

\bibitem[Salomon et~al.(2011)Salomon, Audibert, and Alaoui]{Salomon2011TR}
Salomon, A., Audibert, J.-Y., and Alaoui, I.~E.
\newblock Regret lower bounds and extended upper confidence bounds policies in
  stochastic multi-armed bandit problem.
\newblock \emph{arXiv preprint arXiv:1112.3827}, 2011.

\bibitem[Salomon et~al.(2013)Salomon, Audibert, and Alaoui]{Salomon2013JMLR}
Salomon, A., Audibert, J.-Y., and Alaoui, I.~E.
\newblock Lower bounds and selectivity of weak-consistent policies in
  stochastic multi-armed bandit problem.
\newblock \emph{Journal of Machine Learning Research}, 14\penalty0
  (Jan):\penalty0 187--207, 2013.

\bibitem[{Speedometer, {\it
  https://storage.cloud.google.com/speedometer}}()]{speedometer}
{Speedometer, {\it https://storage.cloud.google.com/speedometer}}.

\bibitem[Villar et~al.(2015)Villar, Bowden, Wason,
  et~al.]{Villar2015SS:MAB4Clinical}
Villar, S.~S., Bowden, J., Wason, J., et~al.
\newblock Multi-armed bandit models for the optimal design of clinical trials:
  benefits and challenges.
\newblock \emph{Statistical Science}, 30\penalty0 (2):\penalty0 199--215, 2015.

\bibitem[Walraevens et~al.(2003)Walraevens, Steyaert, and Bruneel]{add_ref_7}
Walraevens, J., Steyaert, B., and Bruneel, H.
\newblock Performance analysis of a single-server atm queue with a priority
  scheduling.
\newblock \emph{Computers \& Operations Research}, 30\penalty0 (12):\penalty0
  1807 -- 1829, 2003.

\bibitem[Wu et~al.(2015)Wu, Srikant, Liu, and Jiang]{Wu2015NIPS:CCB}
Wu, H., Srikant, R., Liu, X., and Jiang, C.
\newblock Algorithms with logarithmic or sublinear regret for constrained
  contextual bandits.
\newblock In \emph{The 29th Annual Conference on Neural Information Processing
  Systems (NIPS)}, Montr\'eal, Canada, Dec. 2015.

\bibitem[Zhou(2015)]{Zhou2015CMAB:Survey}
Zhou, L.
\newblock A survey on contextual multi-armed bandits.
\newblock \emph{arXiv preprint arXiv:1508.03326}, 2015.

\end{thebibliography}

\clearpage
\newpage
\pagebreak

\appendix

\section*{Appendices}

\section{AdaUCB under Dirac Rewards} \label{app:analysis_dirac}

\subsection{Proof of Lemma~\ref{thm:adaUCB_dirac_bounds}}

\textbf{1) Upper Bound}

We can verify that the conclusion holds for $t = 1$ and $2$, because $C_2(1) = 0$, $C_2(2) = 1$, but $\frac{\alpha \log \tau}{\Delta^2}+1 \geq 1$. For $t \geq 3$, we show the result by contradiction.
If the conclusion is false, then there exists a $\tau \geq 3$ such that  $C_2(\tau -1) \leq \frac{\alpha \log (\tau-1)}{\Delta^2} + 1$ but $C_2(\tau) > \frac{\alpha \log \tau}{\Delta^2} + 1$. Because $\log \tau > \log(\tau -1)$, we know that arm-2 is pulled in slot $\tau$.
Thus, if $L_{\tau} = \epsilon_0$,
\begin{equation}
u_1 +\sqrt{\frac{\alpha \log \tau}{C_1(\tau-1)}} \leq u_2 + \sqrt{\frac{\alpha \log \tau}{C_2(\tau-1)}},
\end{equation}
indicating that $\Delta = u_1 - u_2 < \sqrt{\frac{\alpha \log \tau}{C_2(\tau -1)}}$, and thus $C_2(\tau -1) < \frac{\alpha \log \tau}{\Delta^2}$.
Then $C_2(\tau) \leq C_2(\tau-1)+1 < \frac{\alpha \log \tau}{\Delta^2}+1$. Similarly, if $L_{\tau} = 1-\epsilon_1$, we have $C_2(\tau-1) < \frac{\alpha \epsilon_1 \log \tau}{(1-\epsilon_0)\Delta^2} < \frac{\alpha \log \tau}{\Delta^2}$
and $C_2(\tau) \leq C_2(\tau-1)+1 < \frac{\alpha \log \tau}{\Delta^2}+1$. This contradicts the definition of $\tau$ and completes the proof of Part 1.

\textbf{2) Lower Bound}

Note that $C_2(2t) \geq 0$. For $t = 1$, $f(1) = \int_2^1 \min(h'(s),1){\rm d}s - h(2) = - \int_1^2 \min(h'(s),1){\rm d}s - h(2)	< -h(2) < 0$, by noting that $h'(s)>0, \forall s\in[1,2]$.
For $t = 2$, $f(2) = -h(2) < 0$. Thus, the conclusion is true when $t = 1$ and $t = 2$. For $t \geq 3$, we prove the conclusion by contradiction. If the conclusion is false, then there exists $\tau \geq 3$,
such that $C_2(2(\tau-1)) \geq f(\tau-1)$ but $C_2(2\tau) < f(\tau)$. Noticing that for $s > 2$, we have $f'(s) \in [0,1]$, and thus $f(\tau) \leq f(\tau-1) + 1$. Therefore,  $C_2(2\tau) < C_2(2(\tau-1))+ 1$, indicating that $C_2(2\tau) = C_2(2\tau-1) = C_2(2(\tau-1))$.
Hence, arm-1 is pulled at $2\tau-1$ and $2\tau$. In particular, at time $2\tau$, we have
\begin{equation}
u_1 +\sqrt{\frac{\alpha \log (2\tau)}{2\tau-1 - C_2(2\tau-1)}} \geq u_2 + \sqrt{\frac{\alpha \log (2\tau)}{C_2(2\tau-1)}}.
\end{equation}
This implies that
\begin{eqnarray}
\frac{\Delta}{\sqrt{\alpha \log \tau}} &\geq& \frac{\Delta}{\sqrt{\alpha \log (2\tau)}}  \nonumber \\
&\geq& \frac{1}{\sqrt{C_2(2\tau-1)}} - \frac{1}{\sqrt{2\tau-1-C_2(2\tau-1)}} \nonumber
\end{eqnarray}
On the other hand, one can easily show that for a sufficiently large $\tau$, we have $C_2(2\tau-1) \leq (2\tau-1)/2$ and thus $\frac{1}{\sqrt{2\tau-1-C_2(2\tau-1)}} \leq \frac{\sqrt{2}}{\sqrt{2\tau-1}}$. Consequently,
\begin{eqnarray}
\frac{\Delta}{\sqrt{\alpha \log \tau}} \geq \frac{1}{\sqrt{C_2(2\tau-1)}} - \frac{\sqrt{2}}{\sqrt{2\tau-1}}.
\end{eqnarray}
Thus, $C_2(2\tau) = C_2(2\tau-1) \geq h(\tau)$. Moreover, $h(\tau) \geq f(\tau)$ because $\min\{h'(t), 1\} \leq h'(t)$ and $h(2) > 0$, indicating $C_2(2\tau) \geq f(\tau)$. This contradicts the definition of $\tau$ and completes the proof of Part 2.

\subsection{Proof of Theorem~\ref{thm:adaUCB_dirac_regret}}

To prove Theorem~\ref{thm:adaUCB_dirac_regret}, we first show the following lemma using the lower bound in Lemma~\ref{thm:adaUCB_dirac_bounds} , which indicates that we have sufficient explorations when $L_t = 1 - \epsilon_1$ for sufficiently large $t$, and AdaUCB will always pull the better arm when $L_t = 1- \epsilon_1$ after a certain time,

\begin{lemma} \label{thm:adaUCB_dirac_exploit}
In the opportunistic bandit with periodic square wave load and Dirac rewards, there exists a constant $T_1$ independent of $T$ such that  under AdaUCB, $a_t = 1$ when $L_t = 1-\epsilon_1$ for $t \geq T_1$.
\end{lemma}
\begin{proof} 
According to the load we considered here, $L_t = 1-\epsilon_1$  when $t$ is an odd number, and thus $t$ can be represented as $t = 2\tau + 1$, where $\tau = (t-1)/2$. With the setting of  $0 \leq \epsilon_0, \epsilon_1 < 0.5$, we have $\frac{1-\epsilon_0}{\epsilon_1} > 1$. On the other hand, we can verify that in Part 2 of Lemma~\ref{thm:adaUCB_dirac_bounds}, $h'(\tau) < 1$ for sufficiently large $\tau$, indicating that $f(\tau) = h(\tau) - 2h(2) =   \frac{\alpha }{\Delta^2}\big[\log \tau \big(1+\sqrt{\frac{2\alpha\log \tau}{(2\tau-1)\Delta^2}}\big)^{-2} - 2h(2)\Delta^2/\alpha\big]$. Moreover, noting that $\frac{\log(2\tau+1)}{\log\tau + O(1)} \to 1$ and $\frac{\log \tau}{2\tau-1} \to 0$ as $\tau \to \infty$, 
we know that there exists a number $T_1$ such that for $t = 2\tau + 1 \geq T_1$,
\begin{equation}
\frac{\log (2\tau+1)}{\log \tau \bigg(1 + \sqrt{\frac{2\alpha\log \tau}{(2\tau -1)\Delta^2}} \bigg)^{-2}  - \frac{2h(2) \Delta^2}{\alpha}}< \frac{1-\epsilon_0}{\epsilon_1}.
\end{equation}
Thus, according to Part 2 of Lemma~\ref{thm:adaUCB_dirac_bounds}, we have $C_2(2\tau) \geq f(\tau) = \frac{\alpha}{\Delta^2}\big[\log \tau \big(1 + \sqrt{\frac{2\alpha\log \tau}{(2\tau -1)\Delta^2}} \big)^{-2}  - \frac{2h(2) \Delta^2}{\alpha}\big] $, and thus for $t \geq T_1$,
\begin{eqnarray}
\hat{u}_2(t) = u_2 + \sqrt{\frac{\alpha\epsilon_1\log(2\tau+1)}{(1-\epsilon_0)C_2(2\tau)}} < u_1 < \hat{u}_1(t).
\end{eqnarray}
This implies the conclusion of this lemma, i.e., $a_t = 1$ when $L_t = 1-\epsilon_1$ for $t \geq  T_1$.
\end{proof}

The conclusion of Theorem~\ref{thm:adaUCB_dirac_regret} can be obtained by combining Lemma \ref{thm:adaUCB_dirac_exploit} and Part 1 of Lemma~\ref{thm:adaUCB_dirac_bounds}.

\section{AdaUCB under Random Binary-Valued Load and Random Rewards} \label{app:analysis_randomload}

\subsection{Proof of Lemma~\ref{thm:lower_bound_opt_arm_randomload}}

Let  $T_{\epsilon_0}(t) = \sum_{\tau = 1}^t \mathbbm{1}(L_{\tau} = \epsilon_0)$ be the total number of low-load slots up to $t$. Then 
\begin{eqnarray}
&&\mathbb{P}\big\{C_1^{(0)}(t) < \frac{(\rho-\eta) t}{2}\big\} \nonumber \\
& = & \mathbb{P}\big\{C_1^{(0)}(t) < \frac{(\rho -\eta)  t}{2}, T_{\epsilon_0}(t)  < (\rho - \eta) t \big\}  \nonumber \\
&& +  \mathbb{P}\big\{C_1^{(0)}(t) < \frac{(\rho -\eta) t}{2},  T_{\epsilon_0}(t)  \geq (\rho - \eta) t  \big\}.  
\end{eqnarray}
For the first term, we know that
\begin{eqnarray} \label{eq:low_load_slots_bound}
&& \mathbb{P}\big\{C_1^{(0)}(t) < \frac{(\rho -\eta)  t}{2}, T_{\epsilon_0}(t)  < (\rho - \eta) t \big\}  \nonumber \\
&\leq & \mathbb{P}\big\{T_{\epsilon_0}(t)  < (\rho - \eta) t \big\} \leq e^{-2\eta^2 t}.
\end{eqnarray}

For the second term, we note that $C_1^{(0)}(t) < \frac{(\rho -\eta)  t}{2}$  and  $T_{\epsilon_0}(t)  \geq (\rho - \eta) t$  indicate that at least one of the suboptimal arms have been pulled for at least $\lfloor \frac{(\rho -\eta)  t}{2(K-1)} \rfloor$ times in low-load slots. Let $k' > 1$ be the arm that has been pulled for the most times among the suboptimal arms before $t$, and let $t' < t$ be the last slot when arm $k'$ is pulled under $L_{t'} = \epsilon_0$. Obviously, $t' \geq \lfloor \frac{(\rho -\eta)  t}{2(K-1)} \rfloor +K-1$ as $K-1$ slots of the first $K$ slots are assigned to other arms.
Because arm-$k'$ is pulled at time $t'$, we have
\begin{eqnarray}
\hat{u}_{k'}(t') &=& \bar{u}_{k'}(t') + \sqrt{\frac{\alpha \log (t')}{C_{k'}(t'-1)}} \nonumber \\
&\geq& \hat{u}_1(t') = \bar{u}_1(t') + \sqrt{\frac{\alpha \log (t')}{C_1(t'-1)}}.
\end{eqnarray}
This implies that either $\hat{u}_1(t') < u_1$ or $\hat{u}_{k'}(t') \geq u_1$ is true. However, for a fixed $t' \geq \lfloor \frac{(\rho -\eta)  t}{2(K-1)} \rfloor +K-1$, using Chernoff-Hoeffding bound and considering all possible values of $C_{k'}(t'-1)$, we have
\begin{eqnarray}
\mathbb{P}\{ \hat{u}_1(t') < u_1\} &\leq& t' e^{-2\alpha \log t'} = (t')^{-2\alpha + 1}.
\end{eqnarray}
Moreover, because $\frac{\log t}{\lfloor \frac{(\rho -\eta)  t}{2(K-1)} \rfloor - 1} \to 0$ as $t \to \infty$, there exists a $T_2$ independent of $k'$, such that for $t \geq T_2$, $\sqrt{\frac{\alpha\log (t')}{\lfloor \frac{(\rho -\eta)  t}{2(K-1)} \rfloor -1}} \leq \Delta^*/2 \leq \Delta_{k'}/2$.
Thus,
\begin{eqnarray}
\mathbb{P}\{ \hat{u}_{k'}(t')\geq u_1\} &\leq& \mathbb{P}\big\{ \bar{u}_{k'}(t') \geq u_{k'} + \sqrt{\frac{\alpha\log (t')}{C_{k'}(t' - 1)}} \big\} \nonumber \\
     &\leq& (t')^{-2\alpha+1}.
\end{eqnarray}
Considering all possible values of $t'$ and $k'$, we have 
\begin{eqnarray}\label{eq:opt_arm_upper_bound_ovall_all_k}
&&\mathbb{P}\big\{C_1^{(0)}(t) < \frac{(\rho -\eta) t}{2},  T_{\epsilon_0}(t)  \geq (\rho - \eta) t  \big\}  \nonumber \\
&\leq & \sum_{k' = 2} ^ {K} \sum_{t' =  \lfloor \frac{(\rho -\eta)  t}{2(K-1)} \rfloor +K-1}^{t} 2 (t')^{-2\alpha+1} \nonumber \\
&\leq&  \frac{[2(K-1)]^{2\alpha-1}}{2\alpha - 2} \big[(\rho - \eta)t \big]^{-2\alpha + 2}.
 \end{eqnarray}

The conclusion of the lemma then follows by combining Eqs~\eqref{eq:low_load_slots_bound} and \eqref{eq:opt_arm_upper_bound_ovall_all_k}.

\subsection{Proof of Lemma~\ref{thm:random_reward_general_upperbound_randomload}}

We can prove this lemma by borrowing ideas from the analysis for the traditional UCB policies \cite{Auer2002ML:UCB,Bubeck2010PhD:bandits}, except that we need more conditions to bound the UCB for the optimal arm when the load level is $L_t = 1-\epsilon_1$. Specifically, we analyze the probabilities for load levels $\epsilon_0$ and $1-\epsilon_1$, respectively.

\textbf{Case 1: $L_t = \epsilon_0$}

When the load is $L_t = \epsilon_0$, AdaUCB makes decisions according to the indices $\hat{u}_k(t) = \bar{u}_k(t) + \sqrt{\frac{\alpha\log t}{C_k(t-1)}}$ for all $k$'s, which is the same as traditional UCB. Thus, a suboptimal arm $k$ ($k > 1$) can be pulled when at least one of the following events is true:
\begin{equation}\label{eq:better_arm_slot0}
\hat{u}_1(t) \leq u_1,
\end{equation}
\begin{equation}\label{eq:worse_arm_slot0}
\hat{u}_k(t) > u_k + 2\sqrt{\frac{\alpha\log t}{C_k(t-1)}},
\end{equation}
and
\begin{equation}
C_k(t-1) < \frac{4\alpha\log T}{\Delta_k^2}.
\end{equation}
Then we can bound the probabilities for the events \eqref{eq:better_arm_slot0} and \eqref{eq:worse_arm_slot0} according to \cite{Auer2002ML:UCB}.

\textbf{Case 2: $L_t = 1-\epsilon_1$}

When the load is $L_t = 1 - \epsilon_1$, AdaUCB chooses the arm according to the indices  $\hat{u}_k(t) = \bar{u}_k(t) + \sqrt{\frac{\alpha\epsilon_1\log t}{(1-\epsilon_0)C_k(t-1)}}$ for all $k$'s. Thus, pulling the suboptimal arm requires one of the following events is true:
\begin{equation}\label{eq:better_arm_slot1}
\hat{u}_1(t) \leq u_1 - \big(1-\sqrt{\frac{\epsilon_1}{1-\epsilon_0}}\big)\frac{\Delta_k}{2},
\end{equation}
\begin{equation}\label{eq:worse_arm_slot1}
\hat{u}_k(t) > u_k + \big(1+\sqrt{\frac{\epsilon_1}{1-\epsilon_0}}\big)\sqrt{\frac{\alpha\log t}{C_k(t-1)}},
\end{equation}
and
\begin{equation}
C_k(t-1) < \frac{4\alpha\log T}{\Delta_k^2}.
\end{equation}
In fact, if all the above events are false, we have
\begin{eqnarray}
\hat{u}_1(t) &>& u_1 - \big(1-\sqrt{\frac{\epsilon_1}{1-\epsilon_0}}\big)\frac{\Delta_k}{2} \nonumber \\
&=& u_k + \big(1+\sqrt{\frac{\epsilon_1}{1-\epsilon_0}}\big)\frac{\Delta_k}{2} \nonumber \\
&\geq& u_k + \big(1+\sqrt{\frac{\epsilon_1}{1-\epsilon_0}}\big)\sqrt{\frac{\alpha\log t}{C_k(t-1)}} \geq \hat{u}_k(t),\nonumber
\end{eqnarray}
indicating that arm-1 will be pulled. Now we only need to bound the probability of events \eqref{eq:better_arm_slot1} and \eqref{eq:worse_arm_slot1}.

For event \eqref{eq:better_arm_slot1}, since we use a smaller exploration factor than the traditional UCB, we require a sufficiently large $C_1(t-1)$ to guarantee that $\hat{u}_1(t)$ is close enough to $u_1$. Specifically, because $\frac{\log t}{t} \to 0$ as $t \to \infty$, there exists a constant number $T_{3,k}$ such that $\sqrt{\frac{\alpha\log t}{(1-\eta)\rho t/2}} \leq \frac{\Delta_k}{2}$ for $t \geq T_{3,k}$. Thus, for $t \geq T_{3,k}$, we have
\begin{eqnarray}
&&\mathbb{P}\bigg\{\hat{u}_1(t) \leq u_1 - \big(1-\sqrt{\frac{\epsilon_1}{1-\epsilon_0}}\big)\frac{\Delta_k}{2}\bigg\}\nonumber \\
&\leq & \mathbb{P}\bigg\{\hat{u}_1(t) \leq u_1 - \big(1-\sqrt{\frac{\epsilon_1}{1-\epsilon_0}}\big)\frac{\Delta_k}{2}, \nonumber \\
&& ~~C_1(t-1) \geq \frac{(\rho-\eta)  t}{2}\bigg\} \nonumber \\
&& + \mathbb{P}\bigg\{C_1(t-1) <\frac{(\rho-\eta) t}{2} \bigg\} \nonumber \\
&\leq & \mathbb{P}\bigg\{\hat{u}_1(t) \leq u_1 - \big(1-\sqrt{\frac{\epsilon_1}{1-\epsilon_0}}\big)\sqrt{\frac{\alpha\log t}{C_1(t-1)}}\bigg\} \nonumber \\
&& + \mathbb{P}\bigg\{C_1(t-1) < \frac{(\rho-\eta) t}{2}\bigg\}\nonumber \\
&\leq& \mathbb{P}\bigg\{\bar{u}_1(t) \leq u_1 - \sqrt{\frac{\alpha\log t}{C_1(t-1)}}\bigg\} \nonumber \\
&&+ \mathbb{P}\bigg\{C_1(t-1) < \frac{(\rho-\eta) t}{2}\bigg\}.\nonumber
\end{eqnarray}
Using Chernoff-Hoeffding bound and considering all possible values of $C_1(t-1)$, we have
\begin{eqnarray}
&& \mathbb{P}\bigg\{\bar{u}_1(t) \leq u_1 - \sqrt{\frac{\alpha\log t}{C_1(t-1)}}\bigg\} \leq t ^{-2\alpha +1}.\nonumber
\end{eqnarray}
Moreover, we can bound $\mathbb{P}\big\{C_1(t-1) < \frac{(\rho-\eta) t}{2}  \big\}$ according to Lemma~\ref{thm:lower_bound_opt_arm_randomload}. Thus
\begin{eqnarray}
&&\mathbb{P}\bigg\{\hat{u}_1(t) \leq u_1 - \big(1-\sqrt{\frac{\epsilon_1}{1-\epsilon_0}}\big)\frac{\Delta_k}{2}\bigg\} \nonumber \\
&\leq&  t ^{-2\alpha +1}+ e^{-2\eta^2 t} + \frac{[2(K-1)]^{2\alpha-1}}{2\alpha - 2} \big[(\rho - \eta)t \big]^{-2\alpha + 2}. \nonumber
\end{eqnarray}

For event \eqref{eq:worse_arm_slot1}, it is equivalent to $\{\bar{u}_k(t) > u_k + \sqrt{\frac{\alpha\log t}{C_k(t-1)}}\}$. Similar to the analysis of $\bar{u}_1(t)$, we have
\begin{eqnarray}
 \mathbb{P}\bigg\{\hat{u}_k(t) > u_k + \big(1+\sqrt{\frac{\epsilon_1}{1-\epsilon_0}}\big)\sqrt{\frac{\alpha\log t}{C_k(t-1)}}\bigg\} \leq t^{-2\alpha+1}.\nonumber
\end{eqnarray}

Combining the above two cases over all $t$, and noticing that 
\begin{eqnarray}
&&\sum_{t=1}^T \! \big[2t ^{-2\alpha+1} \! + e^{-2\eta^2 t} + \frac{[2(K-1)]^{2\alpha-1}}{2\alpha - 2} \big[(\rho - \eta)t \big]^{-2\alpha + 2}\big] \nonumber \\
&\leq & \bigg[1+ \frac{1}{2\alpha - 2} \bigg] + \frac{[2(K-1)]^{2\alpha-1} (\rho - \eta)^{-2\alpha + 2}}{2\alpha - 2} \big(1+ \frac{1}{2\alpha - 3}\big)  \nonumber \\
&& + \frac{e^{-2\eta^2}}{1-e^{-2\eta^2}} 
= O(1),
\end{eqnarray}
we obtain the conclusion of this lemma.

\subsection{Proof of Lemma~\ref{thm:random_reward_lowerbound_randomload}}

Let $t' \leq t$ be the last slot before $t$ when the optimal arm (arm-1) is pulled under load level $L_t = \epsilon_0$.
According to Lemma~\ref{thm:lower_bound_opt_arm_randomload}, we have,
\begin{eqnarray} \label{eq:last_optimal_pull_smaller_randomload}
&&\mathbb{P}\{t' < (\rho-\eta) t/2\} \nonumber \\
&\leq&  \mathbb{P}\big\{C_1^{(0)}(t) < \frac{(\rho-\eta) t}{2}\big\}  \nonumber  \\
&\leq& e^{-2\eta^2 t} + \frac{[2(K-1)]^{2\alpha - 1}}{2\alpha - 2} \big[(\rho - \eta)t \big]^{-2\alpha + 2}.
\end{eqnarray}

Now we focus on the case where $t' \geq (\rho-\eta) t/2$.
Again according to Lemma~\ref{thm:lower_bound_opt_arm_randomload}, we know that $C_1(t') \geq C_1^{(0)}(t') \geq (\rho-\eta) t'/2$ with high probability. Because $\frac{\log t'}{(\rho-\eta) t'/2 - 1}\to 0$, there exists a constant $T_4$, such that for any $t' \geq T_4$, we have $\sqrt{\frac{\alpha \log t'}{C_1(t'-1)}} \leq \delta/2$ if $C_1(t'-1) \geq (\rho-\eta) t'/2-1$. Thus, for $t' \geq T_4$,
\begin{eqnarray}\label{eq:random_opt_arm_div_randomload}
&& \mathbb{P}\{\hat{u}_1(t') > u_1 + \delta\} \nonumber \\
&\leq& \mathbb{P}\{\hat{u}_1(t') > u_1 + 2\sqrt{\frac{\alpha\log t'}{C_1(t'-1)}}, \nonumber \\
&&~~C_1(t'-1) \geq (\rho-\eta) t'/2 -1\} \nonumber \\
&&+ \mathbb{P}\{C_1(t'-1) <  (\rho-\eta) t'/2 -1 \}\nonumber \\
&\leq& (t')^{-2\alpha + 1} + \bigg[e^{-2\eta^2 t'} +  \frac{[2(K-1)]^{2\alpha-1}}{2\alpha - 2} \big[(\rho - \eta)t' \big]^{-2\alpha + 2}\bigg]. \nonumber \\
\end{eqnarray}

Note that $C_k(t) < \frac{\alpha \log t}{4(\Delta_k+\delta)^2}$ indicates that $C_k(t'-1) < \frac{\alpha \log t}{4(\Delta_k+\delta)^2}$. Thus, for a fixed $t' \geq \frac{(\rho - \eta)t}{2}$,
\begin{eqnarray}
&& \mathbb{P}\{C_k(t) < \frac{\alpha \log t}{4(\Delta_k +\delta)^2}, t' \} \nonumber \\
&\leq& \sum_{n: n < \frac{\alpha \log t}{4(\Delta_k+\delta)^2}}\mathbb{P}\{\hat{u}_k(t') \leq \hat{u}_1(t'), C_k(t'-1) = n\} \nonumber \\
&\leq& \mathbb{P}\{\hat{u}_1(t') > u_1 + \delta\}   \nonumber \\
&& + \sum_{n: n < \frac{\alpha \log t}{4(\Delta_k+\delta)^2}}\mathbb{P}\{\hat{u}_k(t') \leq u_1 + \delta, C_k(t'-1) = n\}.\nonumber
\end{eqnarray}
We have already bounded the first term by \eqref{eq:random_opt_arm_div_randomload}. To bound the second term, we note that when $C_k(t'-1) = n  < \frac{\alpha \log t}{4(\Delta_k+\delta)^2}$, we have $\sqrt{\frac{\alpha\log t'}{C_k(t'-1)}} \geq 2(\Delta_k+\delta) \sqrt{\frac{\log t'}{\log t}} \geq (2-\delta)(\Delta_k+ \delta)$ for a sufficiently large $t$. Thus, for $n < \frac{\alpha \log t}{4(\Delta_k +\delta)^2}$,
\begin{eqnarray} \label{eq:random_subopt_arm_div_randomload}
&& \mathbb{P}\{\hat{u}_k(t') \leq u_1 + \delta, C_k(t'-1) = n\} \nonumber \\
&\leq&  \mathbb{P}\{\bar{u}_k(t') + \sqrt{\frac{\alpha\log t}{C_k(t'-1)}}\leq u_k + \Delta_k + \delta, \nonumber \\
& & \qquad C_k(t'-1) = n\} \nonumber \\
&\leq & \mathbb{P}\{\bar{u}_k(t') -u_k \leq -(1 - \frac{1}{2-\delta})\sqrt{\frac{\alpha\log t'}{C_k(t'-1)}}  \} \nonumber \\
&\leq & t' e^{-2(\frac{1-\delta}{2-\delta})^2\alpha \log t'} = (t')^{-2\alpha(\frac{1-\delta}{2-\delta})^2 + 1}.
\end{eqnarray}
Combining \eqref{eq:random_opt_arm_div_randomload} and \eqref{eq:random_subopt_arm_div_randomload}, we have
\begin{eqnarray} \label{eq:random_subopt_arm_larger_tprime}
&&\mathbb{P}\{C_k(t) < \frac{\alpha \log t}{4(\Delta_k +\delta)^2}, t' \geq \frac{(\rho - \eta)t}{2} \}\nonumber \\
&\leq&  \sum_{t' = \frac{(\rho-\eta)t}{2}}^{t}  \bigg[(t')^{-2\alpha+1} + e^{-2\eta^2 t'}  \nonumber \\
&&  +  \frac{[2(K-1)]^{2\alpha - 1}}{2\alpha - 2} \big[(\rho - \eta)t' \big]^{-2\alpha + 2} + (t')^{-2\alpha(\frac{1-\delta}{2-\delta})^2 + 1} \bigg] \nonumber \\
&\leq& \frac{1}{2\alpha -2}\big[\frac{(\rho - \eta)t}{2}\big]^{-2\alpha + 2}  + \frac{e^{-\eta^2(\rho-\eta)t}}{1-e^{-2\eta^2 }} \nonumber \\
&& + \frac{[2(K-1)]^{2\alpha-1}(\rho - \eta)^{-4\alpha + 5}}{(2\alpha -2)(2\alpha -3)}(t/2)^{-2\alpha + 3}  \nonumber \\
&&+  \frac{1}{2\alpha (\frac{1-\delta}{2-\delta})^2-2}\big[\frac{(\rho - \eta)t}{2}\big]^{-2\alpha (\frac{1-\delta}{2-\delta})^2 + 2}
\end{eqnarray} 

The conclusion then follows by combining Eqs~\eqref{eq:last_optimal_pull_smaller_randomload} and \eqref{eq:random_subopt_arm_larger_tprime}, and ignoring the lower order terms.

\subsection{Proof of Theorem~\ref{thm:adaUCB_random_regret_randomload}}

Note that we have bounded the total number of pulls of suboptimal arms by Lemma~\ref{thm:random_reward_general_upperbound_randomload}. To show this theorem, we only need to show that the number of pulls of suboptimal arms under the higher load, i.e., $L_t = 1-\epsilon_1$, is bounded by $O(1)$. To show this result, we analyze the probability of pulling an suboptimal arm $k > 1$ when $L_t = 1-\epsilon_1$.

We can easily verify that to pull the suboptimal arm $k$ under load $L_t = 1-\epsilon_1$, at least one of the following events should be true:
\begin{equation}
\hat{u}_1(t) \leq u_k  + \frac{3\Delta_k}{4},
\end{equation}
\begin{equation}
\hat{u}_k(t) > u_k + \frac{3\Delta_k}{4}.
\end{equation}

To bound the probability of the event about $\hat{u}_1(t)$, we note that $C_1(t-1) > (\rho-\eta) t/2$  with high probability and for a sufficient large $t$, we have $(1-\sqrt{\frac{\epsilon_1}{1-\epsilon_0}})\sqrt{\frac{\alpha\log t}{C_1(t-1)}} \leq \frac{\Delta_k}{4}$. Then, using a similar argument as in Lemma~\ref{thm:random_reward_general_upperbound_randomload}, we can bound the first event as:
\begin{eqnarray}
&&\mathbb{P}\big\{\hat{u}_1(t) \leq u_k  + \frac{3\Delta_k}{4}\big\} \nonumber \\
&\leq&   t ^{-2\alpha+1} + e^{-2\eta^2 t} + \frac{[2(K-1)]^{2\alpha - 1}}{2\alpha - 2} \big[(\rho - \eta)t \big]^{-2\alpha + 2}.\nonumber
\end{eqnarray}

To bound the probability of the event about $\hat{u}_k(t)$, we let $\delta \in (0,\Delta^*)$ be a number satisfying $\alpha > 16(1+\frac{\delta}{\Delta^*})^2$ and $\sqrt{\frac{\epsilon_1}{1-\epsilon_0}} < \frac{\Delta^*}{8(\Delta^* + \delta)}$ (the existence of $\delta$ is guaranteed by the range of $\alpha$ and $\sqrt{\frac{\epsilon_1}{1-\epsilon_0}}$).
Note that when $C_k(t-1) \geq \frac{\alpha\log (t-1)}{4(\Delta_k+\delta)^2}$, which occurs with high probability according to Lemma~\ref{thm:random_reward_lowerbound_randomload}, we have $\sqrt{\frac{\alpha\log t}{\zeta C_k(t-1)}} \leq \Delta^*/2\leq \Delta_k/2$ (where $\zeta = 16(1+\delta/\Delta^*)^2$) and $\sqrt{\frac{\epsilon_1\alpha\log t}{(1-\epsilon_0)C_k(t-1)}} \leq \Delta^*/4$. This leads to the following bound:
\begin{eqnarray}
&& \mathbb{P}\{\hat{u}_k(t) > u_k + \frac{3\Delta_k}{4}\} \nonumber \\
&\leq & \mathbb{P}\big\{\hat{u}_k(t) > u_2 + \frac{3\Delta_k}{4}, C_k(t-1) \geq \frac{\alpha\log (t-1)}{4(\Delta_k+\delta)^2}\big\} \nonumber \\
&& + \mathbb{P}\big\{C_k(t-1) < \frac{\alpha\log (t-1)}{4(\Delta_k+\delta)^2}\big\}.
\end{eqnarray}
The second term is bounded according to Lemma~\ref{thm:random_reward_lowerbound_randomload}. For the first term, we have
\begin{eqnarray}
&& \mathbb{P}\{\hat{u}_k(t) > u_k + \frac{3\Delta_k}{4}, C_k(t-1) \geq \frac{\alpha\log (t-1)}{4(\Delta_k+\delta)^2}\}  \nonumber\\
&\leq&  \mathbb{P}\big\{\bar{u}_k(t) > u_k + \sqrt{\frac{\alpha\log t}{\zeta C_k(t-1)}} \big\} \nonumber \\
&\leq& t e^{-2\alpha\log t /\zeta } = t^{-2\alpha/\zeta + 1}.
\end{eqnarray}
Summing over all $t$, we have
\begin{eqnarray} \label{eq:num_pulls_suboptimal_arm}
&& \mathbb{E}[C_k^{(1)}(T)] \nonumber \\
&=& O\bigg(\sum_{t = 1}^T  \big(\big[t ^{-2\alpha + 1} + e^{-2\eta^2 t}  \nonumber \\
&&+ \frac{[2(K-1)]^{2\alpha-1} }{2\alpha - 2}  [(\rho - \eta)t ]^{-2\alpha + 2} \big] \nonumber \\
&& + t^{-2\alpha/\zeta + 1} + [t^{-(2\alpha -3)}  + t^{-(2\alpha(\frac{1-\delta}{2-\delta})^2 - 2)}] \bigg) \nonumber \\
&=&O\bigg( 1 + \frac{1}{2\alpha -2} \nonumber \\
&&+ \frac{[2(K-1)]^{2\alpha-1} (\rho - \eta)^{-2\alpha + 2}}{(2\alpha - 2)(2\alpha - 3)} (1+\frac{1}{2\alpha -3}) \nonumber \\
&&+ \frac{e^{-2\eta^2}}{1- e^{-2\eta^2}} + 1+ \frac{1}{2\alpha/\zeta-2} \nonumber \\
&& + 1+ \frac{1}{2\alpha-4} + 1+ \frac{1}{2\alpha(\frac{1-\delta}{2-\delta})^2-3} \bigg) \nonumber \\
&=& O\bigg( \frac{[2(K-1)]^{2\alpha-1}(\rho - \eta)^{-2\alpha+2}}{(2\alpha -2)(2\alpha - 3)} + \frac{e^{-2\eta^2}}{1- e^{-2\eta^2}}  \nonumber \\
&&+ \frac{1}{2\min\{\frac{\alpha}{\zeta} - 1,\alpha -2, \alpha (\frac{1-\delta}{2-\delta})^2 - \frac{3}{2}\} } \bigg).
\end{eqnarray}

From the definition of $\sigma$ and $\zeta$, we know that they only depends on the value of $\epsilon_0$, $\epsilon_1$, and $\Delta_k$, and satisfies that $\alpha/\zeta > 1$ and $\alpha(\frac{1-\delta}{2-\delta})^2 >3/2$, which guarantee the above bound are finite. Considering all suboptimal arms completes the proof of this theorem. 

\subsection{Impact of $\rho$} \label{subsec:impact_rho}

In the above analysis, we focus on the asymptotic behavior of the regret as $T$ goes to infinity. 
 In addition to the $O(\log T)$ term, the constant term depends on other parameters such as $\rho$, $\epsilon_0$, $\epsilon_1$, and $\Delta_k$'s. We are interested in the impact of load fluctuation, and thus discuss the impact of $\rho$ here. As $\rho \to 1$, AdaUCB will become traditional UCB($\alpha$) at most time, and the regret will get close to traditional UCB($\alpha$). We are more curios on the asymptotic behavior of the constant term as $\rho \to 0$. 

To understand the performance of AdaUCB as $\rho \to 0$, we will analyze \eqref{eq:num_pulls_suboptimal_arm} with more details by refining the bound in Lemma~\ref{thm:random_reward_lowerbound_randomload}. We write down the complete format for the conclusion of Lemma~\ref{thm:random_reward_lowerbound_randomload}, which is the summation of Eqs.~\eqref{eq:last_optimal_pull_smaller_randomload} and \eqref{eq:random_subopt_arm_larger_tprime}:
\begin{eqnarray} \label{eq:random_subopt_arm_larger_tprime}
&&\mathbb{P}\{C_k(t) < \frac{\alpha \log t}{4(\Delta_k +\delta)^2} \}\nonumber \\
&\leq& e^{-2\eta^2 t} + \frac{[2(K-1)]^{2\alpha-1}}{2\alpha - 2} \big[(\rho - \eta)t \big]^{-2\alpha + 2}  \nonumber \\
&& + \frac{1}{2\alpha -2}\big[\frac{(\rho - \eta)t}{2}\big]^{-2\alpha + 2}  + \frac{e^{-\eta^2(\rho-\eta)t}}{1-e^{-2\eta^2 }} \nonumber \\
&& + \frac{[2(K-1)]^{2\alpha-1}(\rho - \eta)^{-4\alpha + 5}}{(2\alpha -2)(2\alpha -3)}(t/2)^{-2\alpha + 3}  \nonumber \\
&&+  \frac{1}{2\alpha (\frac{1-\delta}{2-\delta})^2-2}\big[\frac{(\rho - \eta)t}{2}\big]^{-2\alpha (\frac{1-\delta}{2-\delta})^2 + 2}.
\end{eqnarray} 
We can replace $O([t^{-(2\alpha -3)}  + t^{-(2\alpha(\frac{1-\delta}{2-\delta})^2 - 2)}])$ term in \eqref{eq:num_pulls_suboptimal_arm} with the above equation and recalculate the bounds. We keep the terms depending on $\rho$ and $\eta \in (0, \rho)$, and ignore the lower order term. Then the dependence of the regret on $\rho$ can be captured as: 
\begin{eqnarray} 
&& \varphi_{\alpha, \sigma, \zeta}(\rho, \eta) \nonumber \\
&=& O\big((\rho - \eta)^{-2\alpha+2} + (\rho - \eta)^{-4\alpha+5} + (\rho - \eta)^{-2\alpha(\frac{1-\delta}{2-\delta})^2+2}\big) \nonumber \\
&= & O\big((\rho - \eta)^{-2\alpha+2} + (\rho - \eta)^{-4\alpha+5}\big). \nonumber 
\end{eqnarray}
This shows that as $\rho$ goes to zero, the constant term will go to infinity.  In reality, AdaUCB achieves much smaller regret than traditional UCB and TS algorithms, even for small values of $\rho$ such as $\rho = 0.05$ under binary load and $\rho = 0.001$ under continuous load. Such evaluations establishes empirical guidelines on when to use UCB or AdaUCB. More discussions can be found in Appendix~\ref{app:add_simulations}.

\section{AdaUCB under Continuous Load} \label{app:continuous load}
In this section, we provide comprehensive outlines for the proof of Lemma~\ref{thm:regret_upperbound_continuousload}. For Theorem~\ref{thm:regret_upperbound_continuousload_singlethreshold}, with the mapping described in Sec.~\ref{subsec:continuous_load}, the proof is quite similar to that in the binary-valued load case and is not replicated here.

\textit{Sketch of Proof for Lemma~\ref{thm:regret_upperbound_continuousload}:}  Intuitively, AdaUCB becomes traditional UCB($\alpha$) when $L_t \leq l^{(-)}$. If the lower-truncation threshold $l^{(-)}$ satisfies $P(L_t \leq l^{(-)}) = \rho > 0$, then AdaUCB will accumulate sufficient explorations from the slots with load $L_t \leq l^{(-)}$ and make the right decision, i.e., pull the best arm, under higher load, and achieve logarithmic regret. Specifically, similar to the two cases in the proof of Lemma~\ref{thm:random_reward_general_upperbound_randomload}, we can consider the following two cases to show its regret:

\textbf{Case 1:}  $L_t \leq l^{(-)}$

When $L_t \leq l^{(-)}$, AdaUCB becomes  traditional UCB($\alpha$), and we can bound $\mathbb{E}[C_k(T)]$ for $k > 1$ similar to UCB($\alpha$) as we did in the proof of Lemma~\ref{thm:random_reward_general_upperbound_randomload}.

\textbf{Case 2:}  $L_t > l^{(-)}$

When the load $L_t > l^{(-)}$, we have 
\begin{equation} 
\epsilon_t \triangleq 1- \tilde{L}_t = 1 - \frac{[L_t]_{l^{(-)}}^{l^{(+)}} - l^{(-)}}{l^{(+)} - l^{(-)}} =  \frac{l^{(+)} - [L_t]_{l^{(-)}}^{l^{(+)}} }{l^{(+)} - l^{(-)}}. \nonumber
\end{equation} 
AdaUCB will pull the arm with the largest  $\hat{u}_k(t) = \bar{u}_k(t) + \sqrt{\frac{\alpha\epsilon_t \log t}{C_k(t-1)}}$ .Thus, pulling the suboptimal arm requires one of the following events is true:
\begin{equation}\label{eq:better_arm_slot1_cont_load}
\hat{u}_1(t) \leq u_1 - \big(1-\sqrt{\epsilon_t}\big)\frac{\Delta_k}{2},
\end{equation}
\begin{equation}\label{eq:worse_arm_slot1_cont_load}
\hat{u}_k(t) > u_k + \big(1+\sqrt{\epsilon_t}\big)\sqrt{\frac{\alpha\log t}{C_k(t-1)}},
\end{equation}
and
\begin{equation}
C_k(t-1) < \frac{4\alpha\log T}{\Delta_k^2}.
\end{equation}
Then we can bound the regret using the same argument for  Case 2 in the proof of Lemma~\ref{thm:random_reward_general_upperbound_randomload}.
The total regret are than bound by summing over the above two cases.


\section{Additional Simulations} \label{app:add_simulations}

\subsection{Simulations under Synthetic Scenarios}\label{app:synthetic_simu}

We first use the periodic square wave load to study the impact of load fluctuation. The square wave load is as defined in \ref{subsec:Dirac} with $l^{(-)} = \epsilon_0$ and $l^{(+)} = 1-\epsilon_1$. In this case,  the load fluctuation  is larger when $\epsilon_0$ and $\epsilon_1$ are smaller. As we can see from Fig.~\ref{fig:square_wave}, with adaptive exploration based on load level, AdaUCB significantly reduces the regret in opportunistic bandits, especially compared with UCB$(\alpha)$. Comparing the results for different $\epsilon_0$ and $\epsilon_1$ values, we can see that the improvement of AdaUCB is more significant under loads with larger fluctuations. In particular, when $\epsilon_0 = 0$, i.e., the exploration cost is zero if the load is under certain threshold, the proposed AdaUCB achieves $O(1)$ regret, as shown in Fig.~\ref{fig:square_wave_0}.

\begin{figure*}[thbp]
	\begin{center}
		\begin{minipage}[t]{\textwidth}
			\begin{center}
				\subfigure[$\epsilon_0 = \epsilon_1 = 0$]{\includegraphics[angle = 0,height = 0.23\linewidth,width = 0.32\linewidth]{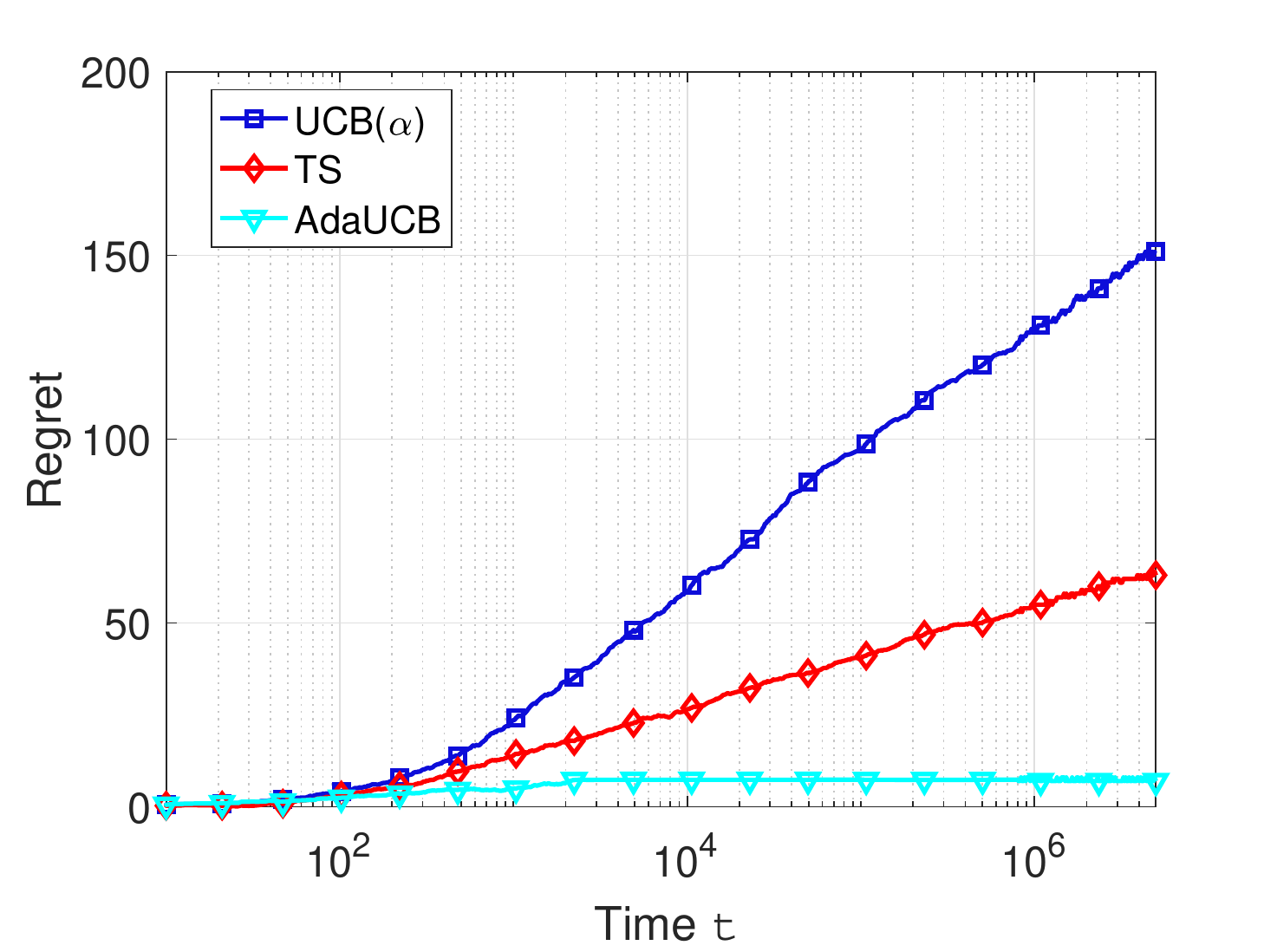}
					\label{fig:square_wave_0}}
				\subfigure[$\epsilon_0 = \epsilon_1 = 0.05$]{\includegraphics[angle = 0,height = 0.23\linewidth,width = 0.32\linewidth]{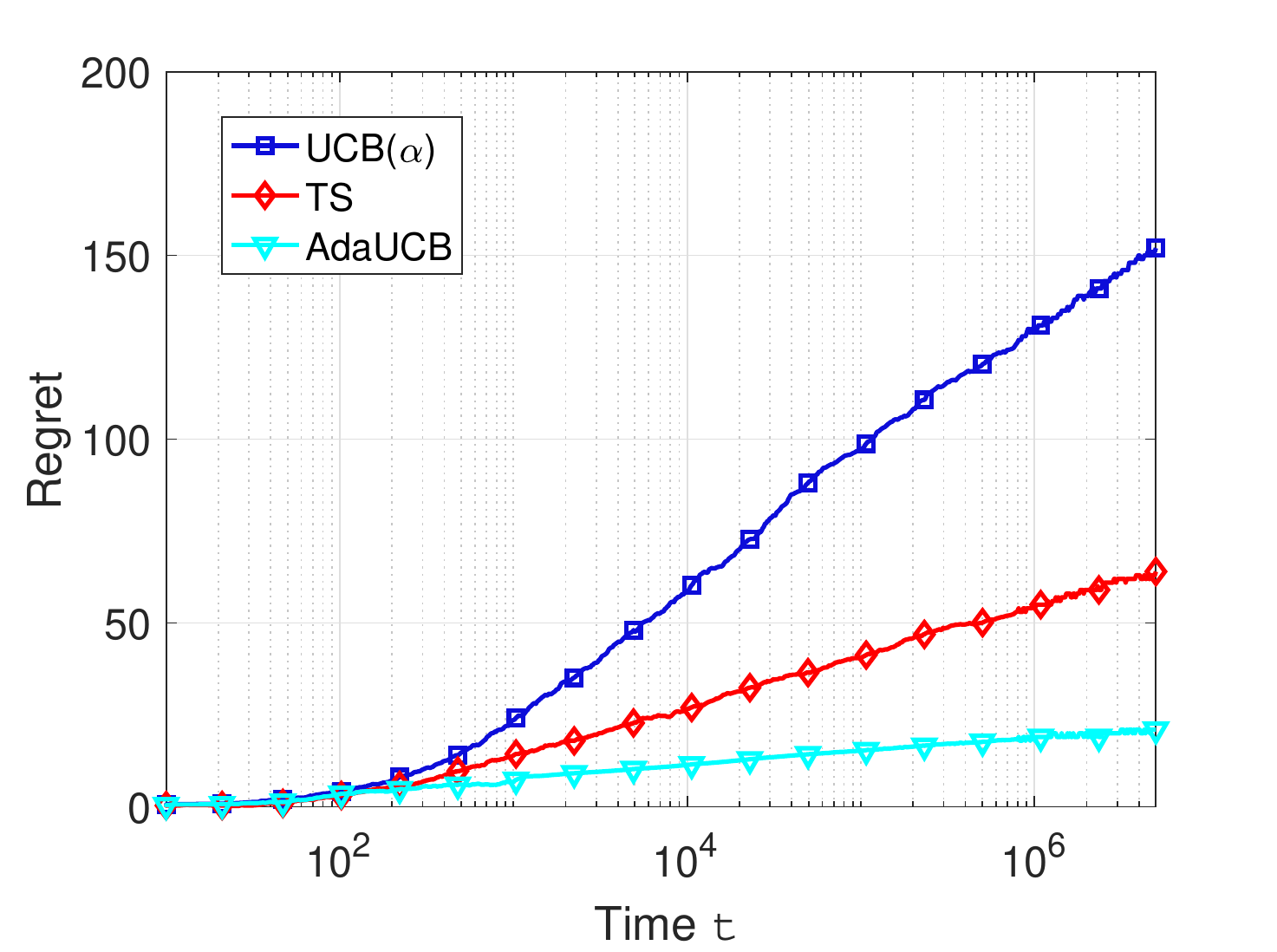}
					\label{fig:square_wave_0p05}}
				\subfigure[$\epsilon_0 = \epsilon_1 = 0.1$]{\includegraphics[angle = 0,height = 0.23\linewidth,width = 0.32\linewidth]{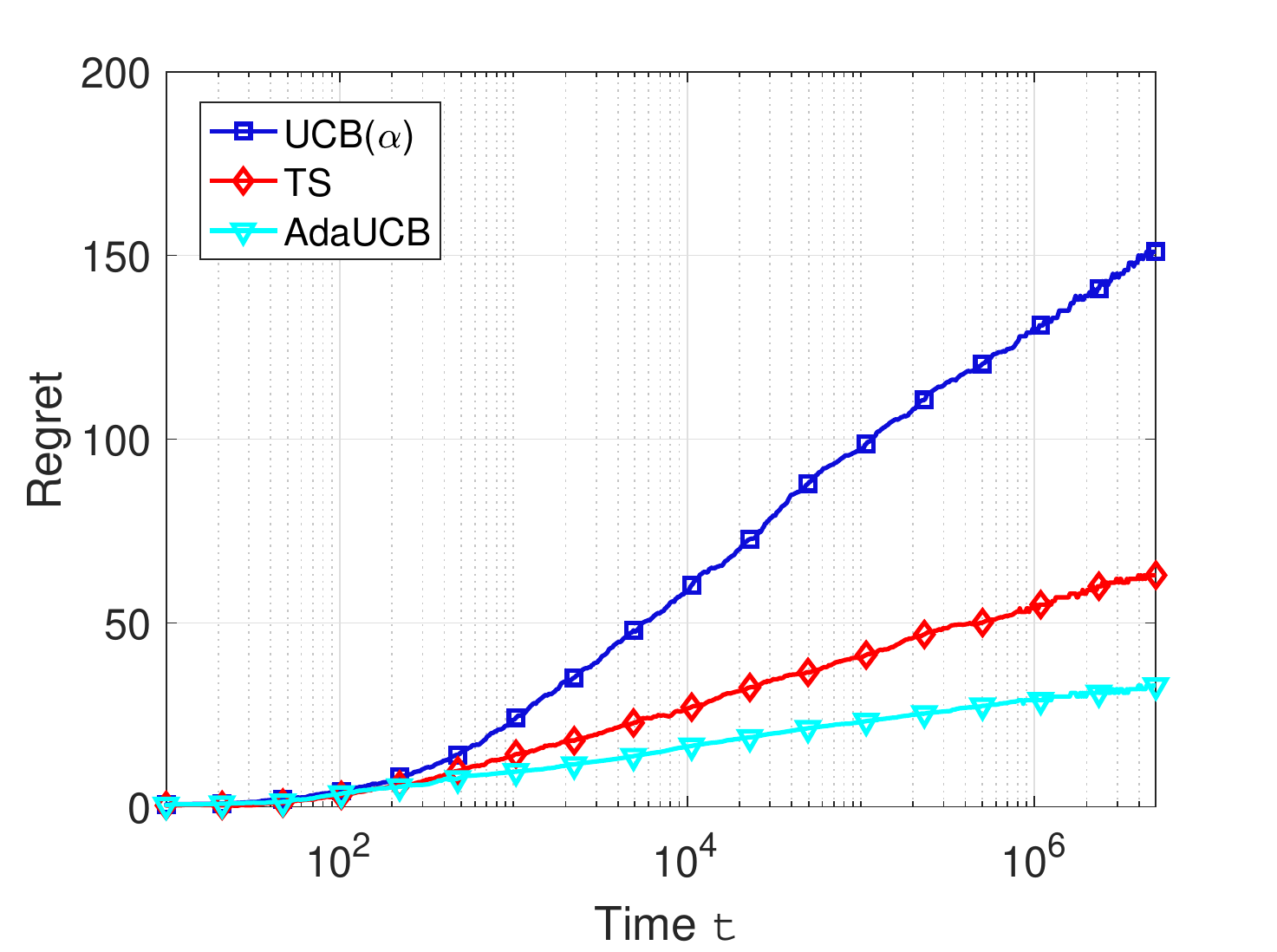}
					\label{fig:square_wave_0p1}}
				\vspace{-.25cm}
				\caption{Regret under periodic square wave load.}
				\label{fig:square_wave}
			\end{center}
		\end{minipage}
	\end{center}
\end{figure*}

We also test the performance of the algorithms under binary-valued load, which is the load in Sec. \ref{subsec:random}. 
The results are shown in Fig.~\ref{fig:bernoulli_load}.
Here, let $\epsilon_0 = \epsilon_1 = 0$. It can be seen that the AdaUCB achieves $O(1)$ regret with respect to $T$ which is consistent with the analytical result. 
To evaluate the asymptotic performance of AdaUCB as $\rho$ gets close to zero, we evaluate its performance under smaller $\rho$ in Fig.~\ref{fig:binary_rho}. We do see that for very small $\rho$, AdaUCB achieves larger regret at the beginning stage, but it performs relatively better as $T$ increases. Moreover, AdaUCB achieves much better performance even for a small $\rho$, such as $\rho = 0.05$. In other words, 5\% of low-load slots will provide sufficient opportunities to explore and achieve lower regret.

\begin{figure*}[thbp]
	\begin{center}
		\begin{minipage}[t]{\textwidth}
			\begin{center}
				\subfigure[$\rho = 0.1$]{\includegraphics[angle = 0,height = 0.23\linewidth,width = 0.32\linewidth]{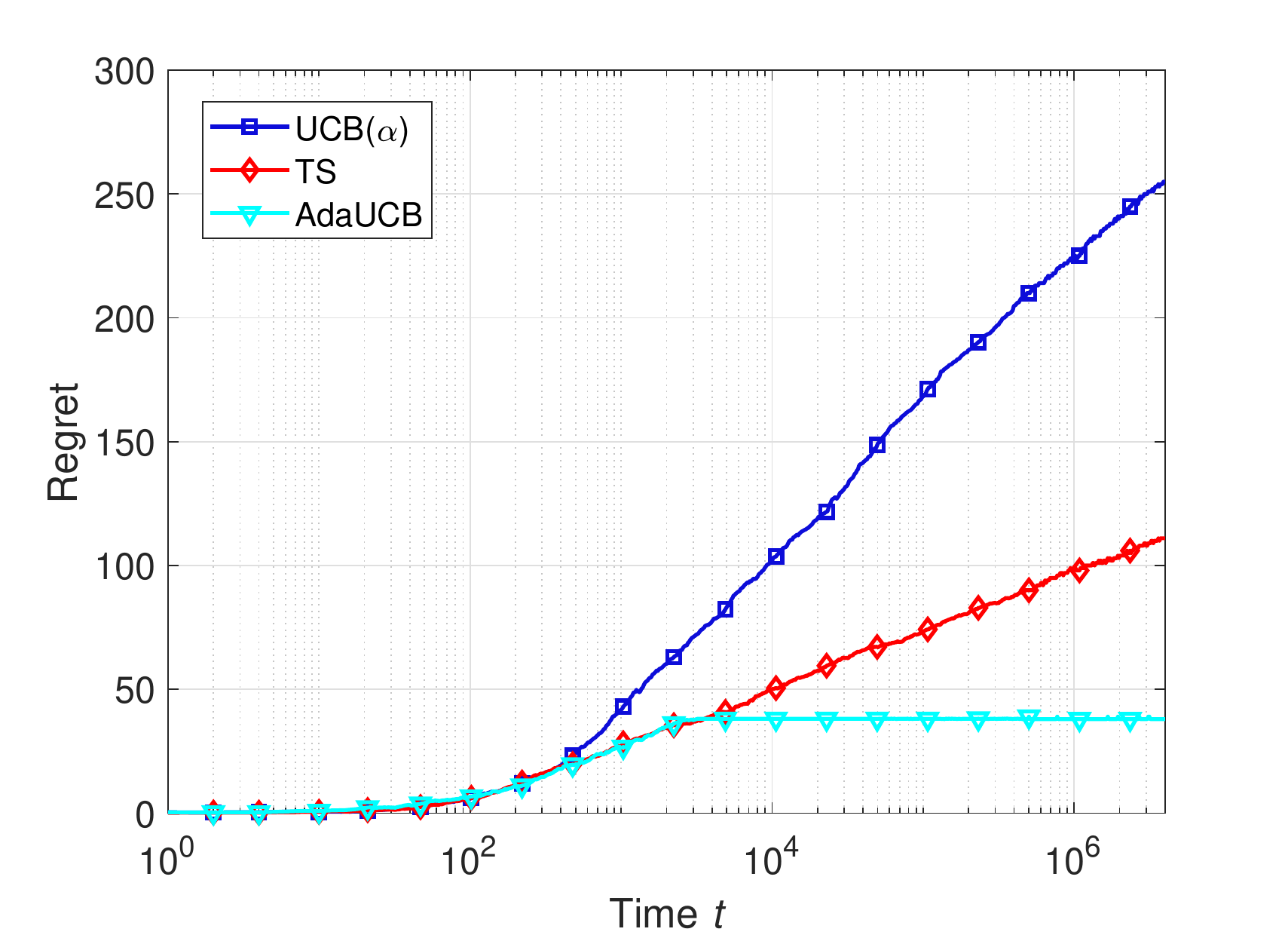}
					\label{fig:LoadBernoulli0p1}}
				\subfigure[$\rho = 0.5$]{\includegraphics[angle = 0,height = 0.23\linewidth,width = 0.32\linewidth]{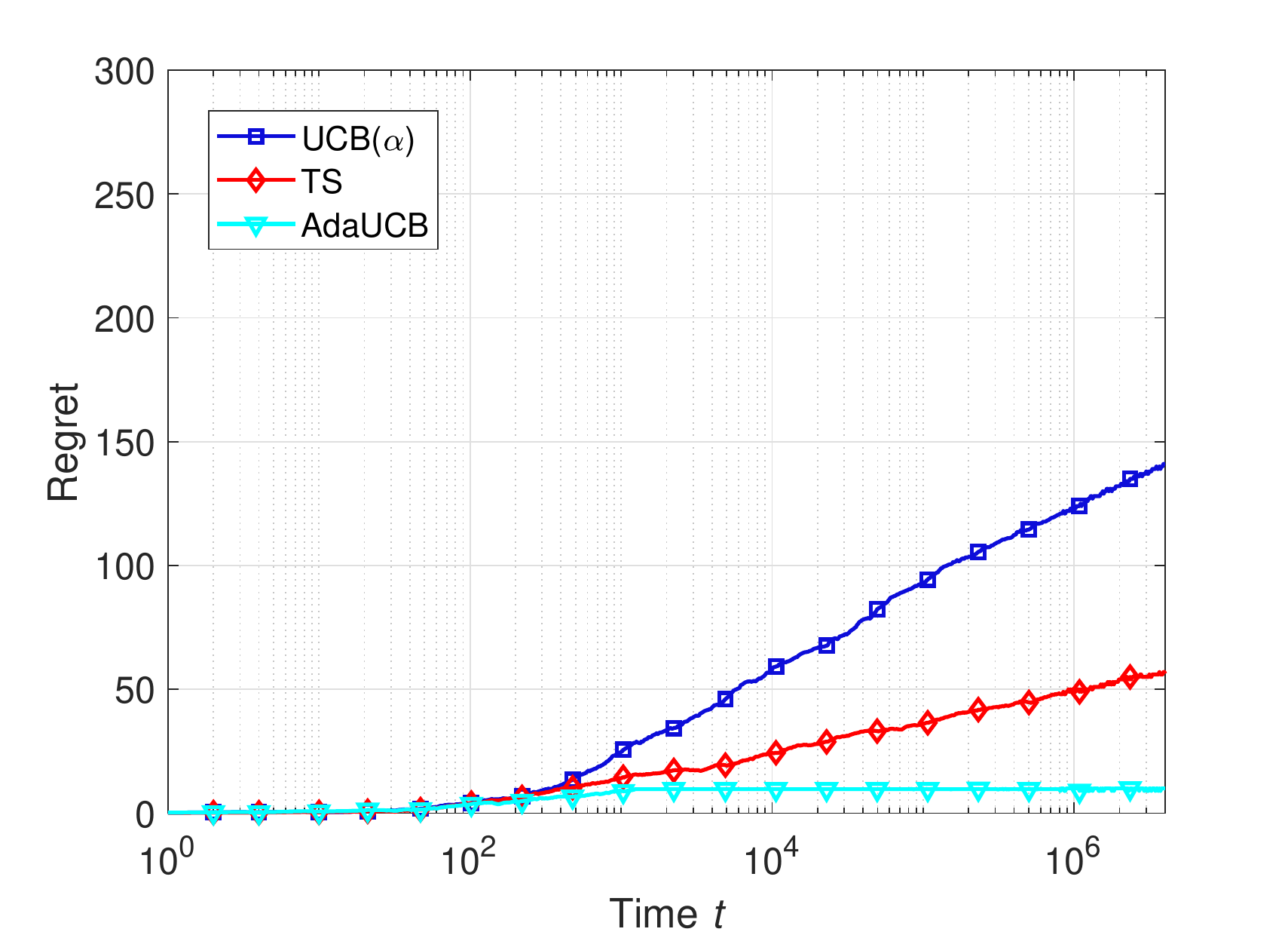}
					\label{fig:LoadBernoulli0p5}}
				\subfigure[$\rho = 0.9$]{\includegraphics[angle = 0,height = 0.23\linewidth,width = 0.32\linewidth]{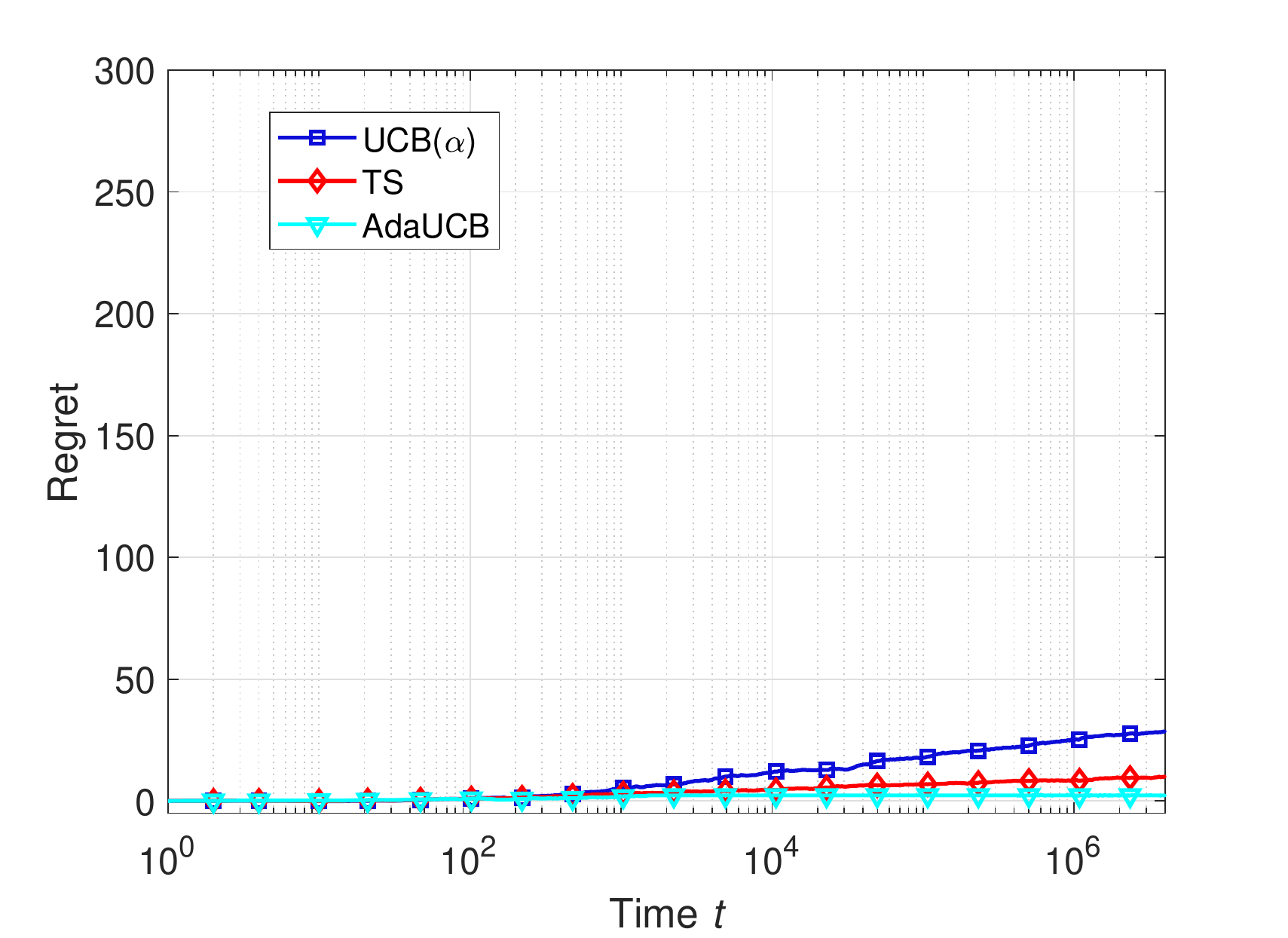}
					\label{fig:LoadBernoulli0p9}}
				\vspace{-.25cm}
				\caption{Regret under binary-valued load.}
				\label{fig:bernoulli_load}
			\end{center}
		\end{minipage}
	\end{center}
\end{figure*}

\begin{figure*}[thbp]
	\begin{center}
		\begin{minipage}[t]{\textwidth}
			\begin{center}
				\subfigure[$\rho = 0.005$]{\includegraphics[angle = 0,height = 0.23\linewidth,width = 0.32\linewidth]{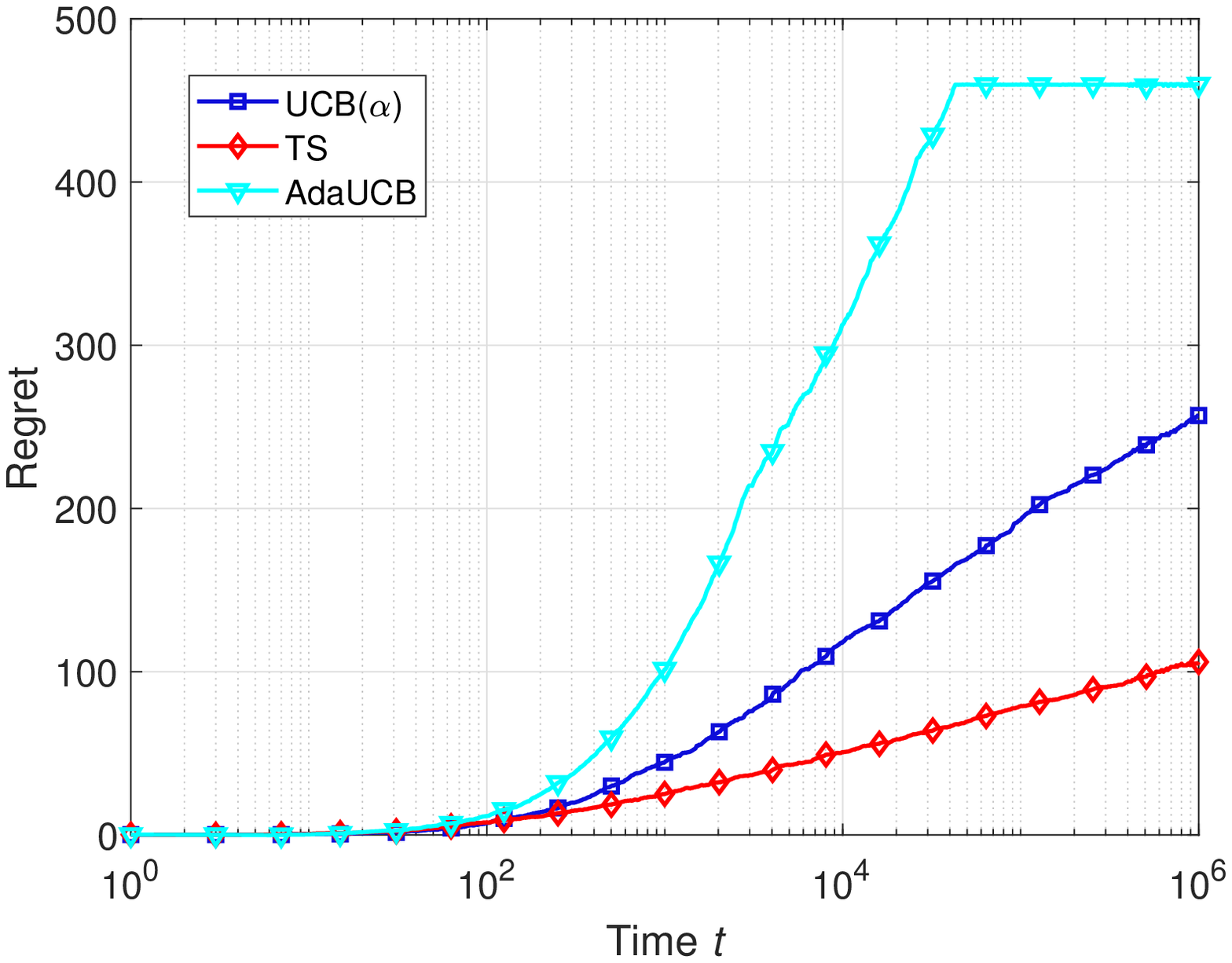}
					\label{fig:binary_rho0p005}}
				\subfigure[$\rho = 0.01$]{\includegraphics[angle = 0,height = 0.23\linewidth,width = 0.32\linewidth]{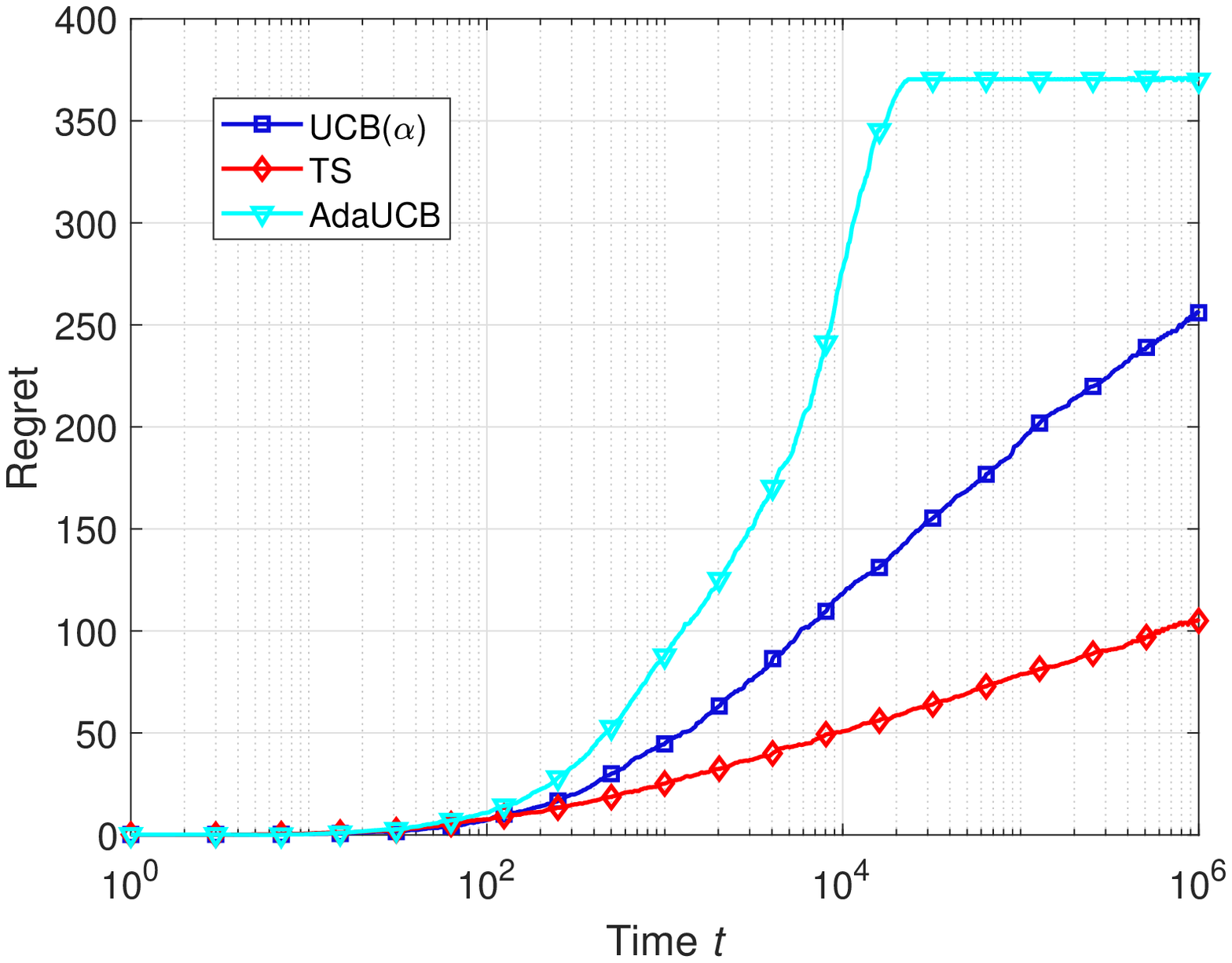}
					\label{fig:binary_rho0p01}}
				\subfigure[$\rho = 0.05$]{\includegraphics[angle = 0,height = 0.23\linewidth,width = 0.32\linewidth]{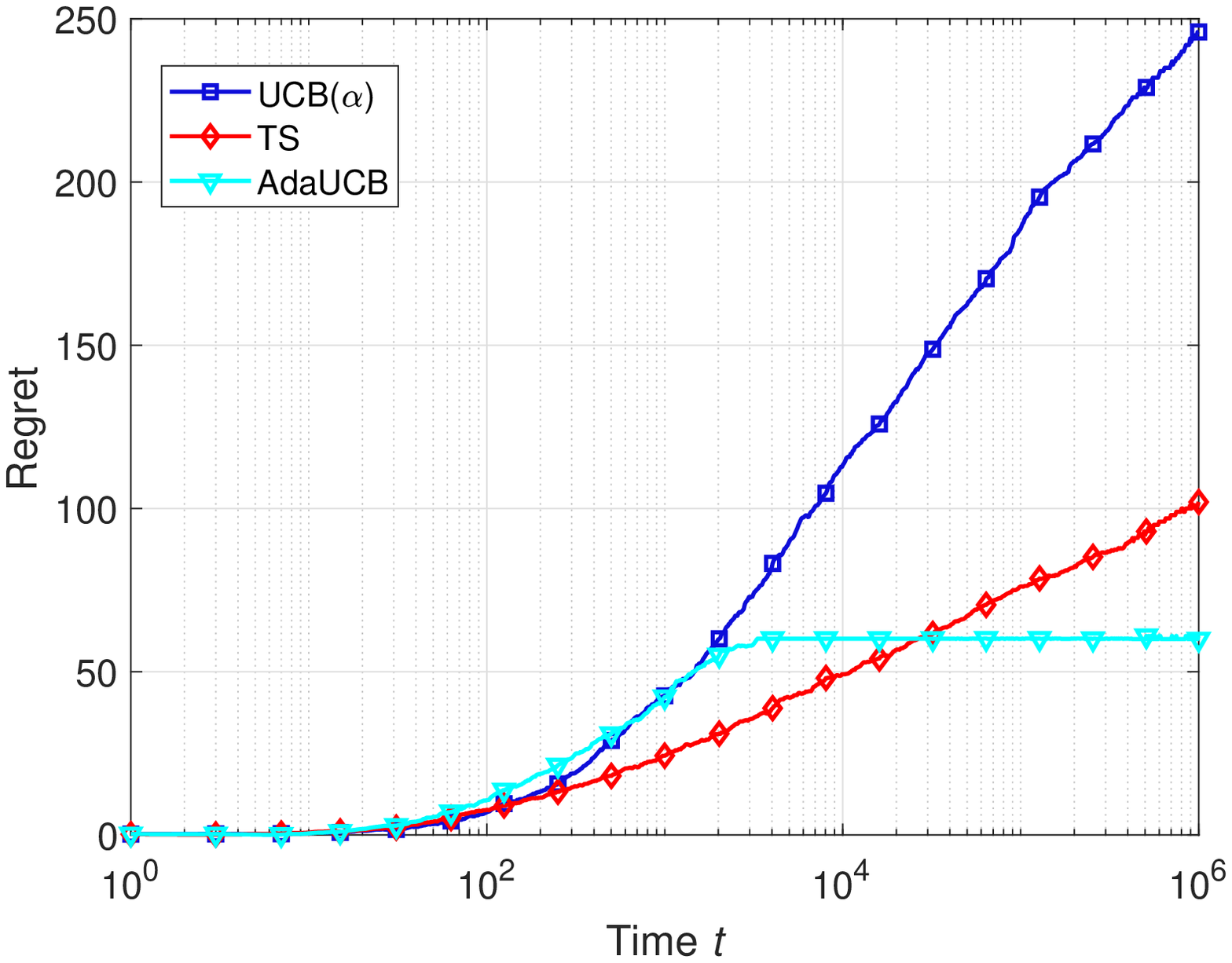}
					\label{fig:binary_rho0p05}}
				\vspace{-.25cm}
				\caption{Regret under binary-valued load with small $\rho$.}
				\label{fig:binary_rho}
			\end{center}
		\end{minipage}
	\end{center}
\end{figure*}

\begin{figure*}[thbp]
\begin{center}
\begin{minipage}[t]{\textwidth}
\begin{center}		
	\subfigure[$l^{(-)} = l^{(-)}_{0.01},  l^{(+)} = l^{(+)}_{0.01}$]{\includegraphics[angle = 0,height = 0.23\linewidth,width = 0.32\linewidth]{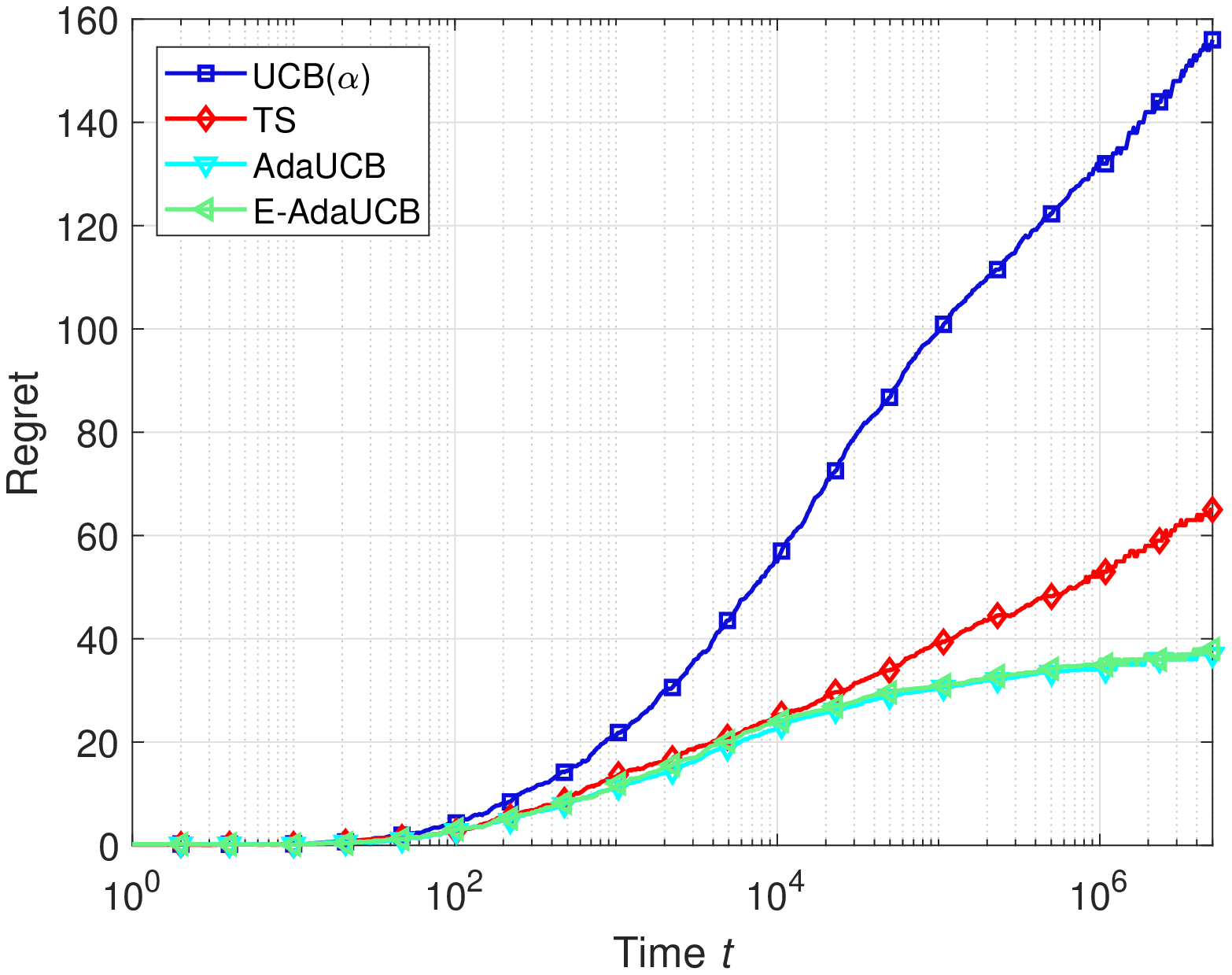}
		\label{fig:beta_load_jump}}
	\subfigure[AdaUCB: $l^{(-)} = l^{(-)}_{0.05},  l^{(+)} = l^{(+)}_{0.05}$;
	\hspace{\textwidth}
	\quad AdaUCB($l^{(-)} =l^{(+)}$): $l^{(+)} =l^{(-)} = l^{(-)}_{0.05}$]{\includegraphics[angle = 0,height = 0.23\linewidth,width = 0.32\linewidth]{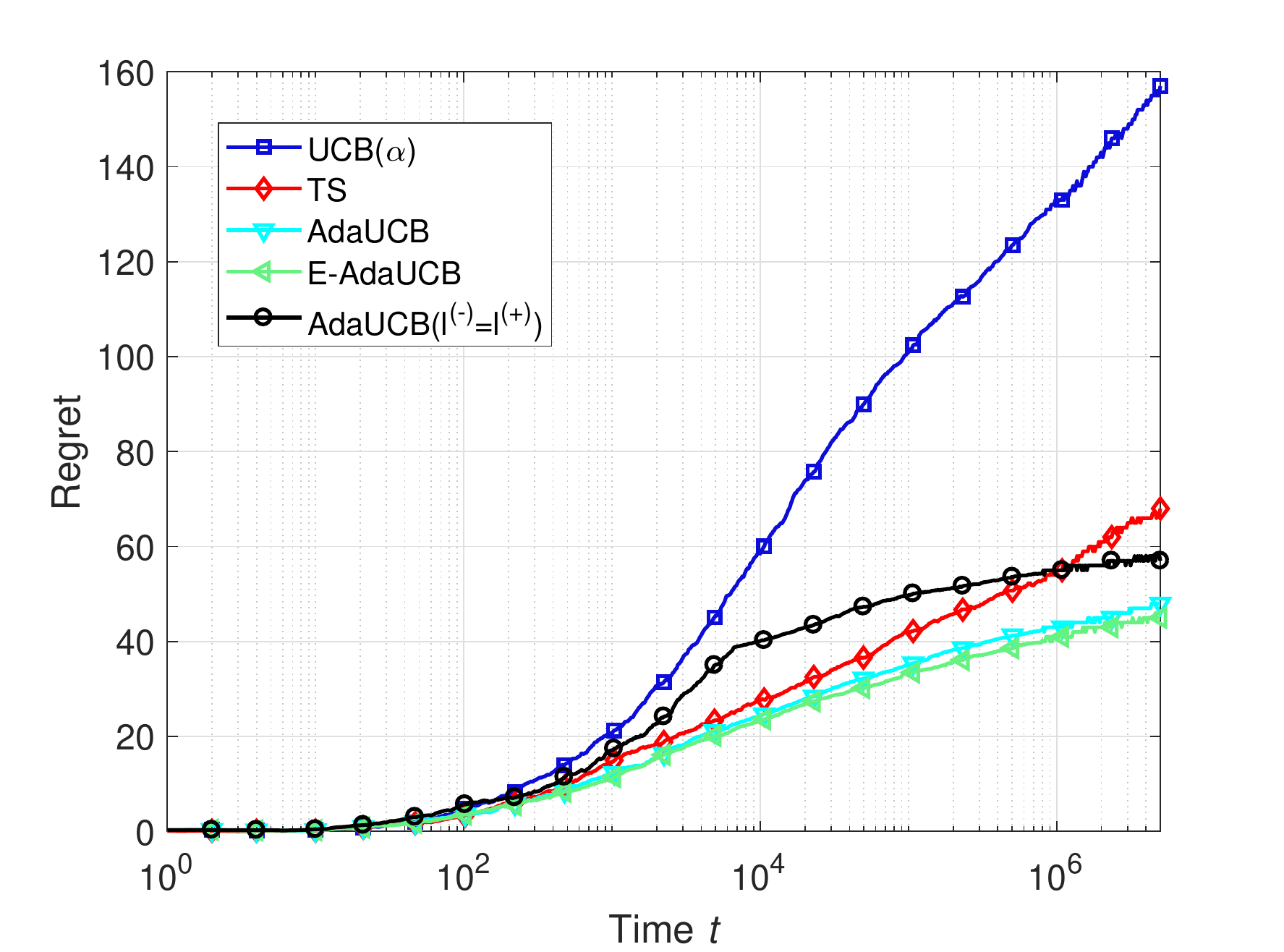}
		\label{fig:beta_load_0p05_0p05}}
	\subfigure[$l^{(-)} = l^{(-)}_{0.1},  l^{(+)} = l^{(+)}_{0.1}$]{\includegraphics[angle = 0,height = 0.23\linewidth,width = 0.32\linewidth]{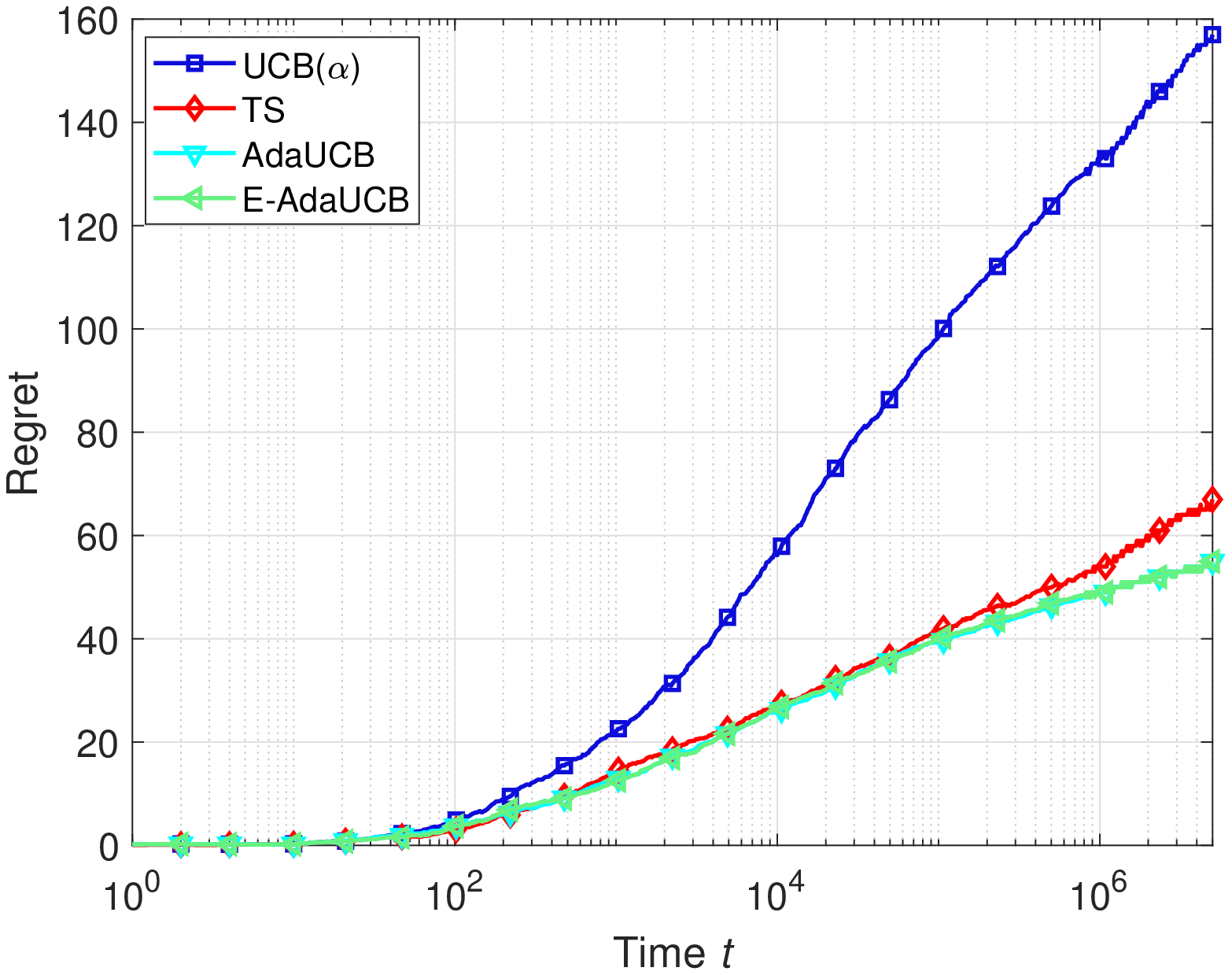}
		\label{fig:beta_load_0p1_0p1}}
	\vspace{-.25cm}
	\caption{Regret under beta distributed load.}
	\label{fig:beta_load1}
\end{center}
\end{minipage}
\end{center}
\end{figure*}

We next investigate the performance of the algorithms under continuous load. Here, we assume the load is i.i.d.~over time, following the beta distribution $Beta(2,2)$. We study the impact of different truncation thresholds for AdaUCB. We define $l^{(-)}_{\rho}$ as the lower threshold such that $\mathbb{P}\{L_t \leq l^{(-)}_{\rho}\}=\rho$, and $l^{(+)}_{\rho}$ as the upper threshold such that $\mathbb{P}\{L_t \geq l^{(+)}_{\rho}\}=\rho$. From Fig.~\ref{fig:beta_load1}, we can see that the selection of $l^{(-)}$ and $l^{(+)}$ affects the performance of AdaUCB.
We also evaluate the impact of $l^{(-)}$ and $l^{(+)}$ separately with the other one fixed in Fig.~\ref{fig:beta_load_fix_l+} and Fig.~\ref{fig:beta_load_fix_l-}, respectively.  
Compared Fig.~\ref{fig:beta_load_fix_l+} to Fig.~\ref{fig:binary_rho}, we can see that the impact of $\rho$ under continuous load is insignificant and the regret of AdUCB is much lower than the traditional UCB algorithm from the beginning stage. This is because under continuous load, exploration also occurs when $L_t \in (l^{(-)}, l^{(+)})$. Also because of this, the impact of $l^{(+)}$ is  negligible as long as there are sufficient percentage of $L_t$ staying in $(l^{(-)}, l^{(+)})$, as shown in Fig.~\ref{fig:beta_load_fix_l-}.

\begin{figure*}[thbp]
	\begin{center}
		\begin{minipage}[t]{\textwidth}
			\begin{center}		
				\subfigure[$l^{(-)} = l^{(-)}_{0.001},  l^{(+)} = l^{(+)}_{0.01}$]{\includegraphics[angle = 0,height = 0.23\linewidth,width = 0.32\linewidth]{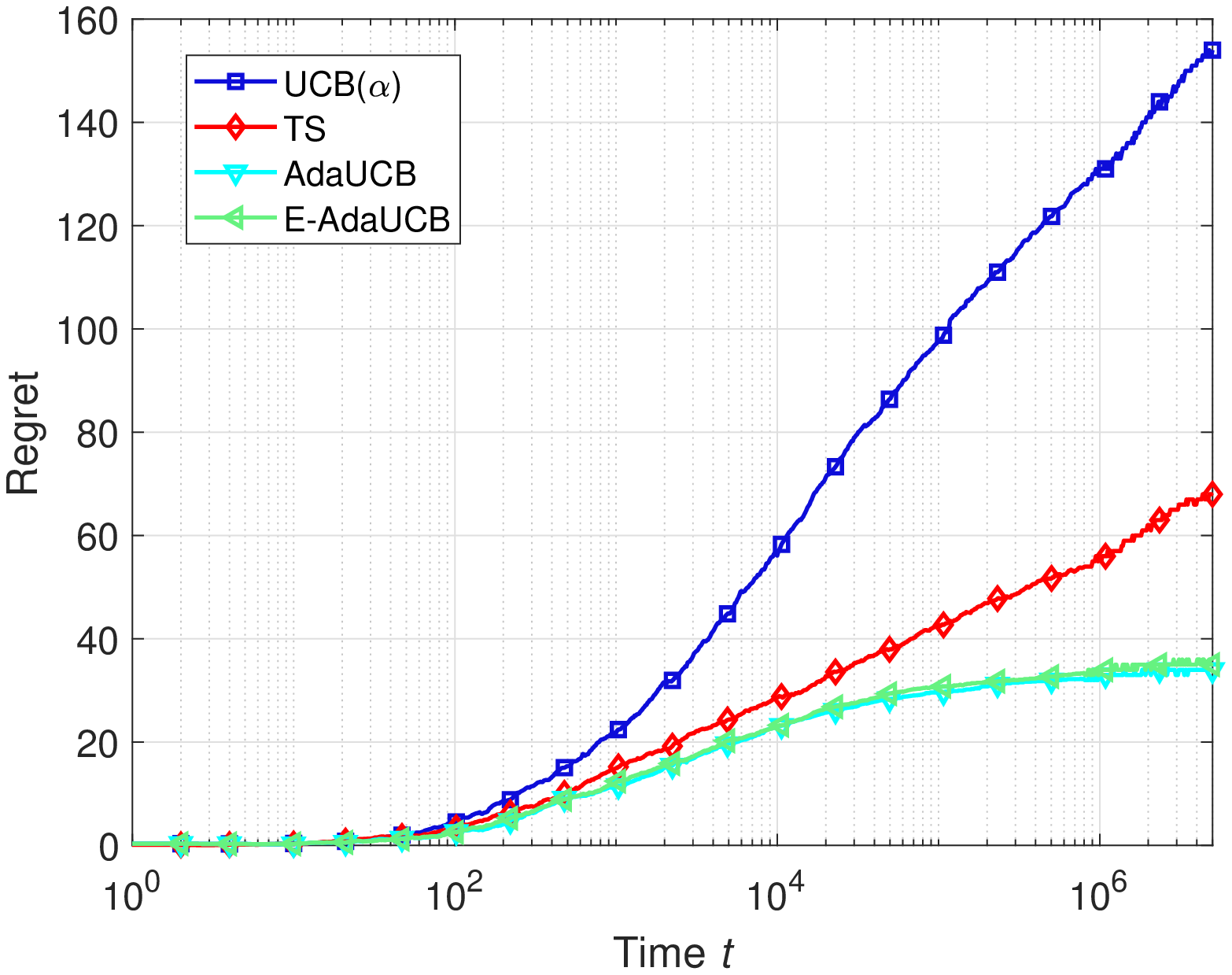}
					\label{fig:beta_load_0p001_fix_l+}}
				\subfigure[$l^{(-)} = l^{(-)}_{0.005},  l^{(+)} = l^{(+)}_{0.01}$]{\includegraphics[angle = 0,height = 0.23\linewidth,width = 0.32\linewidth]{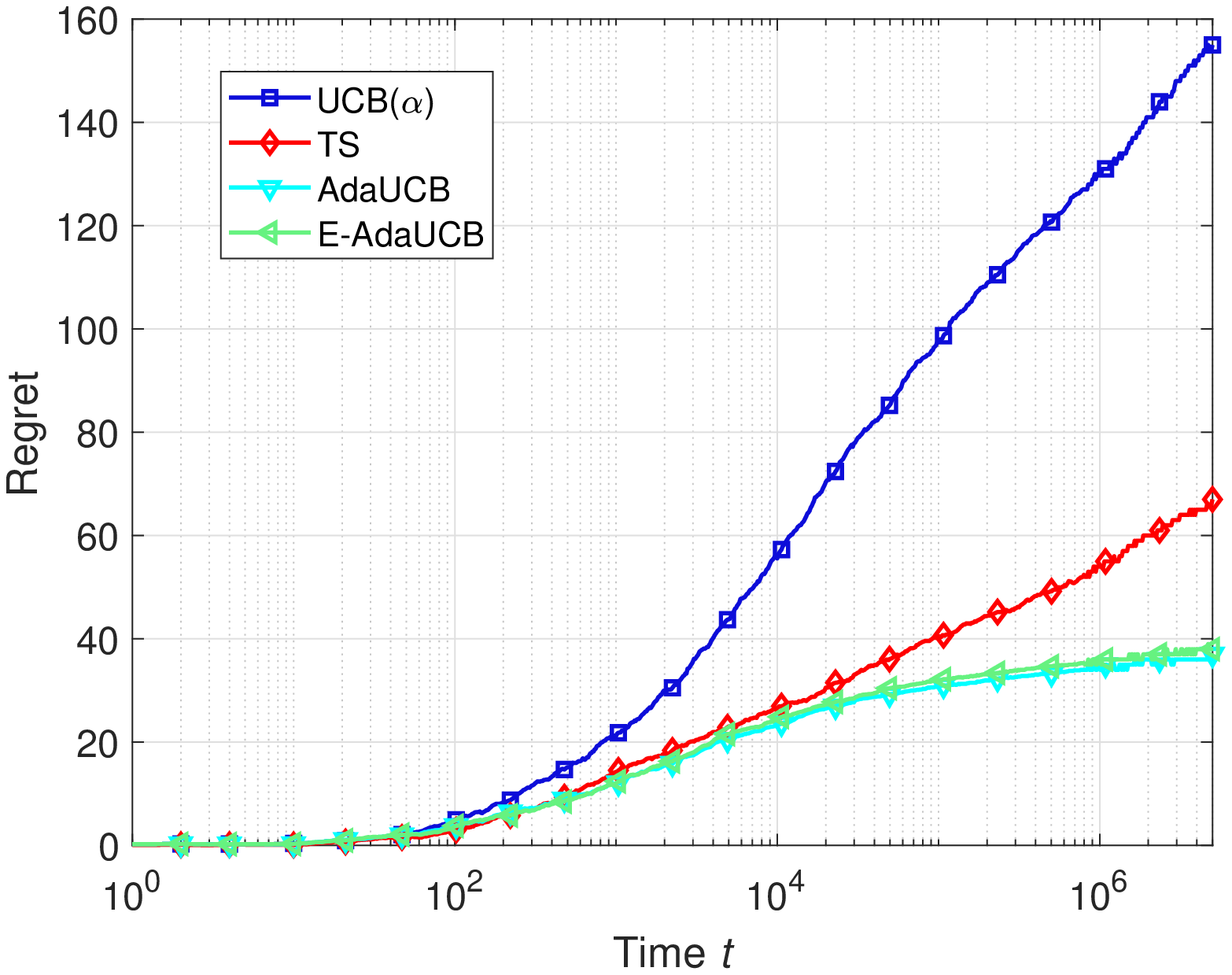}
					\label{fig:beta_load_0p005_fix_l+}}
				\subfigure[$l^{(-)} = l^{(-)}_{0.05},  l^{(+)} = l^{(+)}_{0.01}$]{\includegraphics[angle = 0,height = 0.23\linewidth,width = 0.32\linewidth]{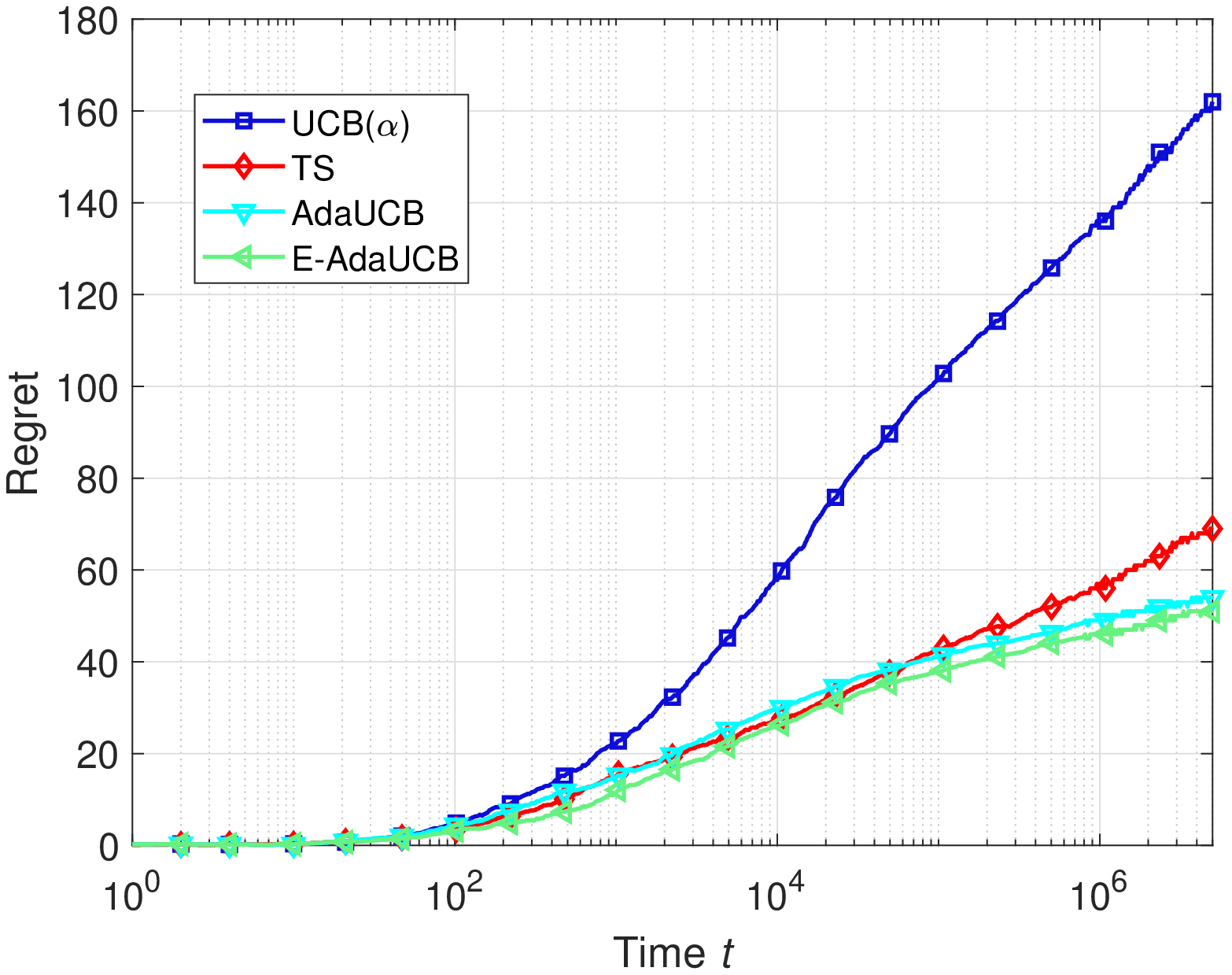}
					\label{fig:beta_load_0p05_fix_l+}}
				\vspace{-.25cm}
				\caption{Regret under beta distributed load with different values of $l^{(-)}$.}
				\label{fig:beta_load_fix_l+}
			\end{center}
		\end{minipage}
	\end{center}
\end{figure*}

\begin{figure*}[thbp]
	\begin{center}
		\begin{minipage}[t]{\textwidth}
			\begin{center}		
				\subfigure[$l^{(-)} = l^{(-)}_{0.01},  l^{(+)} = l^{(+)}_{0.001}$]{\includegraphics[angle = 0,height = 0.23\linewidth,width = 0.32\linewidth]{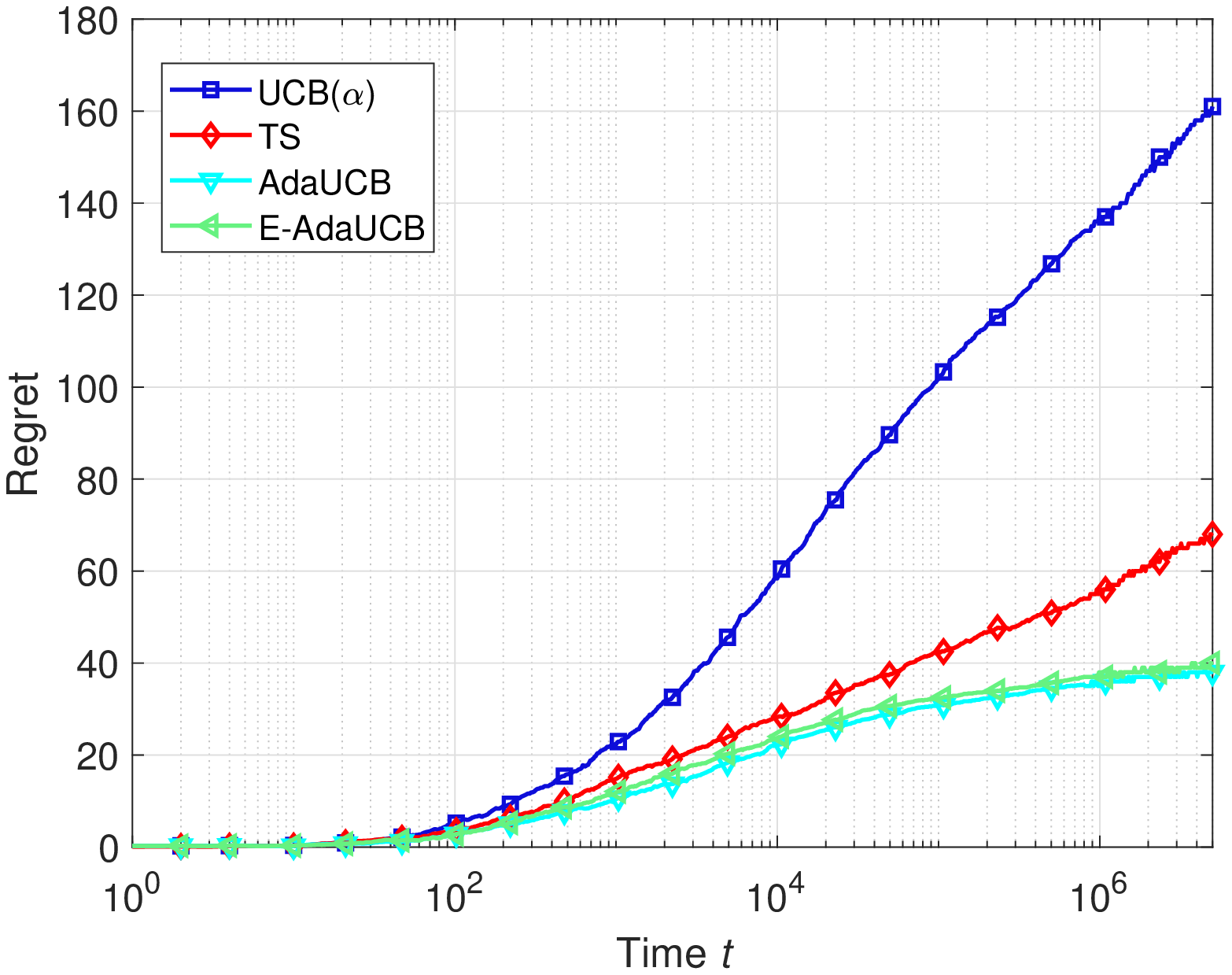}
					\label{fig:beta_load_0p001_fix_l-}}
				\subfigure[$l^{(-)} = l^{(-)}_{0.01},  l^{(+)} = l^{(+)}_{0.005}$]{\includegraphics[angle = 0,height = 0.23\linewidth,width = 0.32\linewidth]{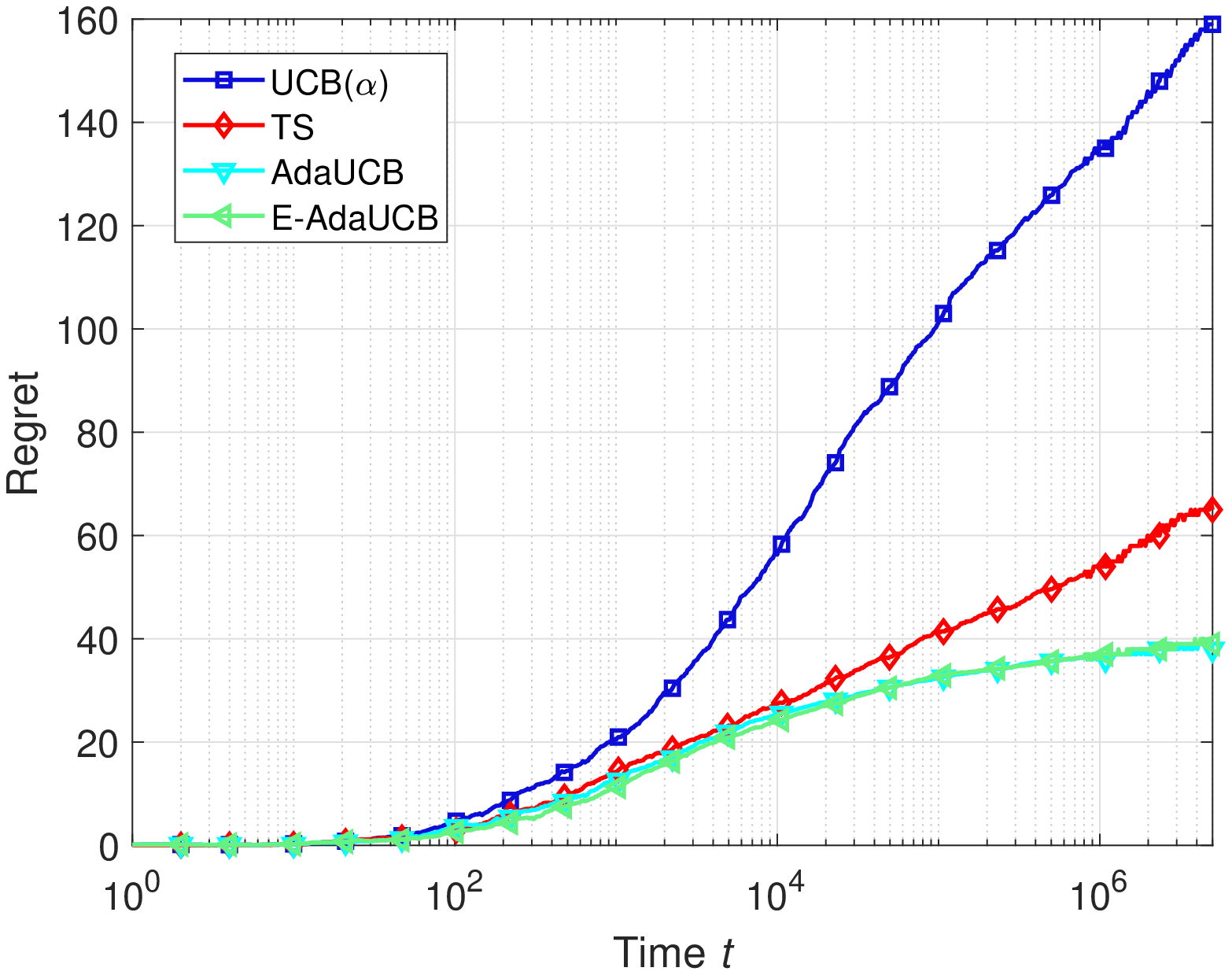}
					\label{fig:beta_load_0p005_fix_l-}}
				\subfigure[$l^{(-)} = l^{(-)}_{0.01},  l^{(+)} = l^{(+)}_{0.05}$]{\includegraphics[angle = 0,height = 0.23\linewidth,width = 0.32\linewidth]{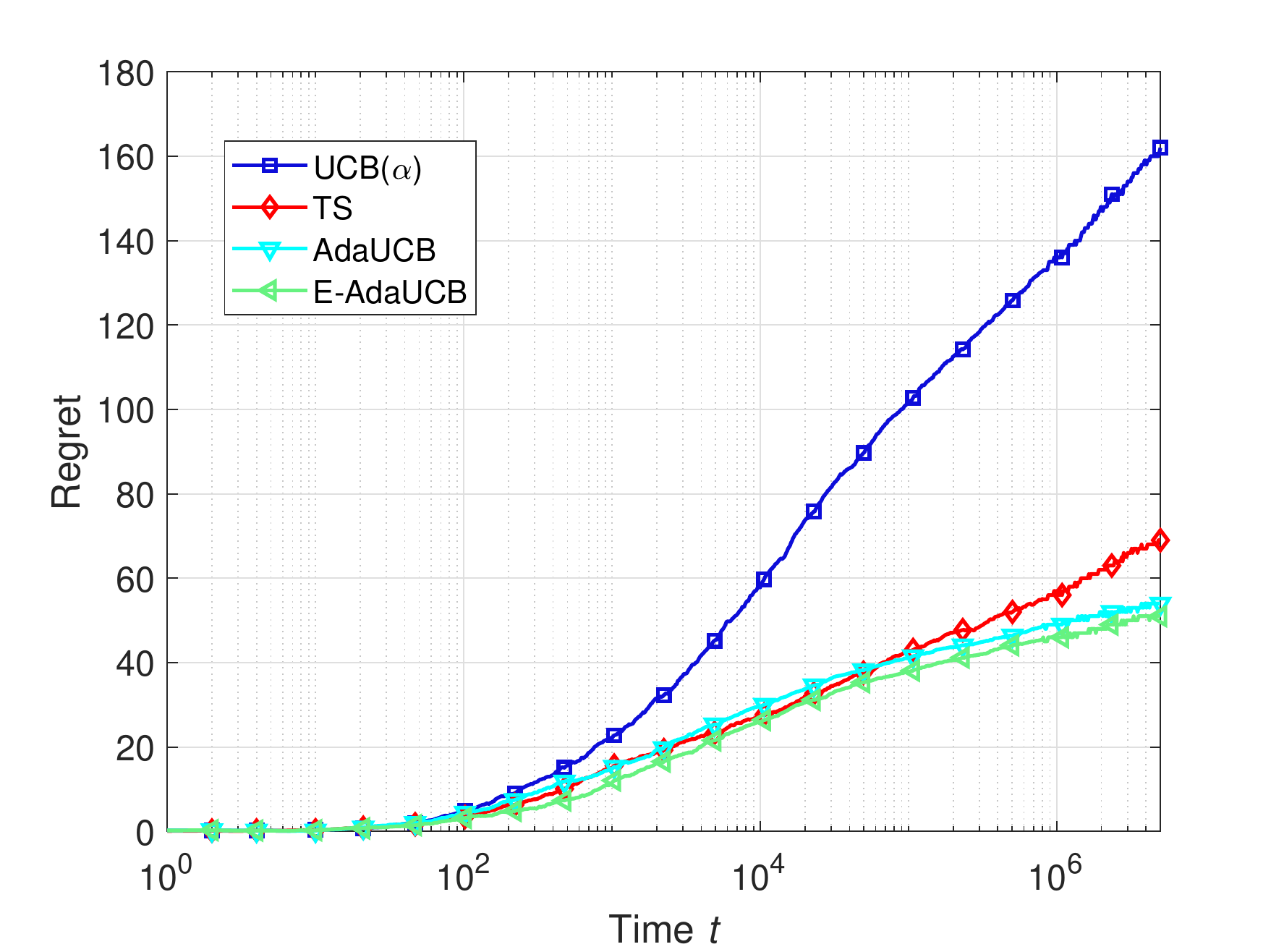}
					\label{fig:beta_load_0p05_fix_l-}}
				\vspace{-.25cm}
				\caption{Regret under beta distributed load with different values of $l^{(+)}$.}
				\label{fig:beta_load_fix_l-}
			\end{center}
		\end{minipage}
	\end{center}
\end{figure*}

\subsection{Simulations Setting in MVNO Systems}\label{app:mvno}

\begin{figure}[thbp]
	\begin{center}
		{\includegraphics[angle = 0,width = 0.9\linewidth]{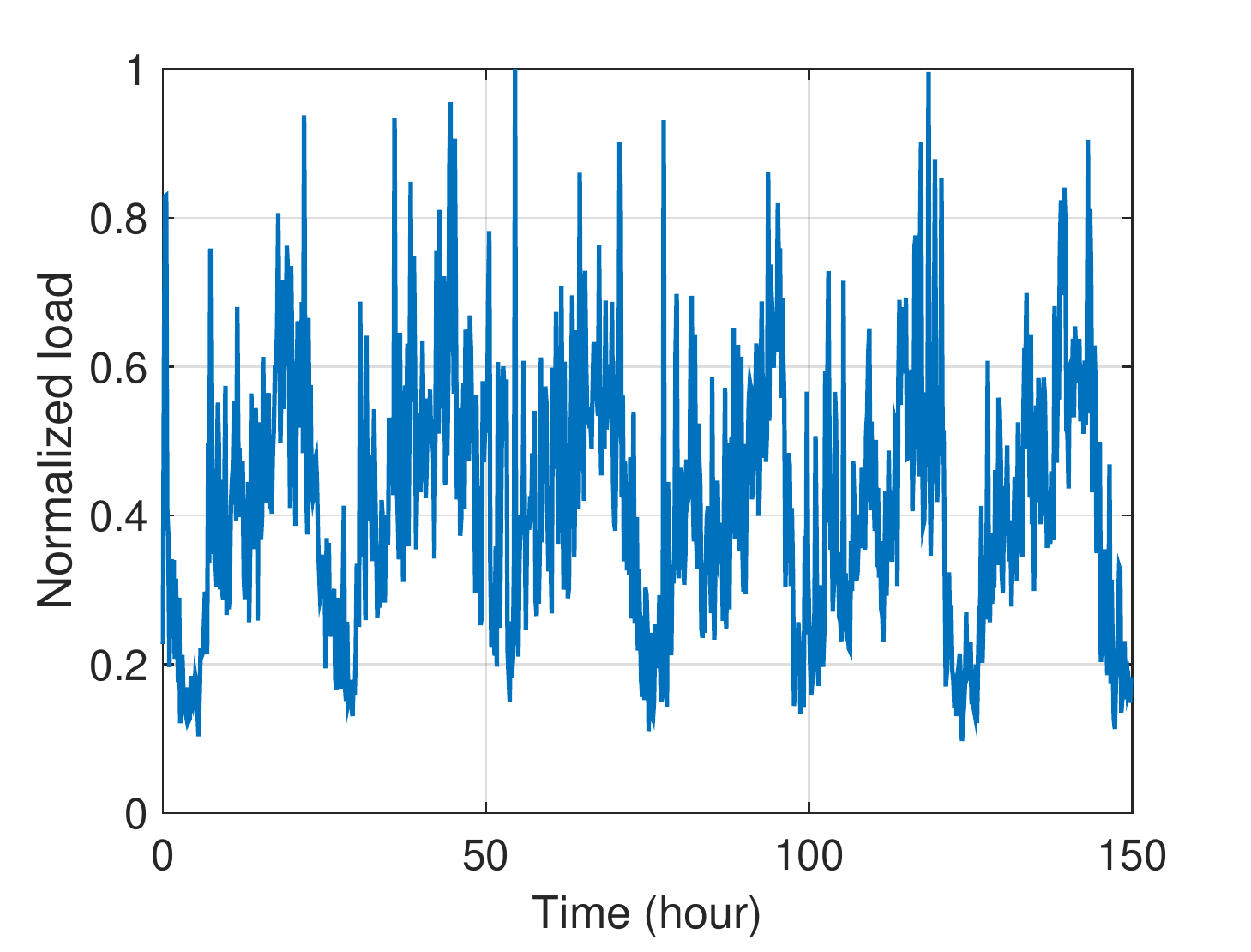}
			\vspace{-.25cm}
			\caption{Normalized traffic load in a cellular network (see Sec.~\ref{sec:sim_results}).}
			\label{fig:traffic_load}}
	\end{center}
	\vspace{-0.25cm}
\end{figure}

Here, we provide more details about the MVNO system mentioned in Sec.~\ref{sec:sim_results}.
In an MVNO system, a virtual operator, such as Google Fi~\cite{googlefi}, provides services to users by leasing  network resources from real mobile operators such as T-Mobile and Sprint. In such a system, the virtual operator would like to provide its users high quality service by accessing the network resources of the real operator with the best network performance. Therefore, we view each real mobile operator as an arm, and the quality of user experienced on that operator network as the reward. We run experiments based on traces collected from real cellular networks, provided by Speedometer~\cite{speedometer}. Speedometer is a custom Android  mobile network measurement app developed by Google, running on thousands of volunteer phones. The data consists of ping, traceroute, DNS lookup, HTTP fetches, and UDP packet-loss measurements for two years. We use round-trip-time (RTT) as a performance indicator for the quality of user experience, and use the inverse of RTT (normalized to [0,1]) as the reward. We consider a three-armed case, where we consider the three operators, Verizon, T-Mobile, and Sprint,  as three arms, using data from Speedometer dataset. We use the load trace of another anonymous operator as the load of the virtual network.The load trace is illustrated in Fig.~\ref{fig:traffic_load}, which shows a clear semi-periodic nature.

\subsection{Compared to Contextual Bandits Algorithms}\label{app:linucb}
Broadly speaking, opportunistic bandits can be considered as a special case of contextual bandits where we can consider the load as context.
To follow this line, we have also compared our algorithms with contextual bandits algorithms. 
We note that by considering load (i.e., $L_t$) as context and actual reward (i.e., $L_t X_{a_t,t}$) as the target for contextual bandits, the opportunistic bandit problem can be formulated as contextual bandits with disjoint linear models, a problem which motivates the design of LinUCB algorithm \cite{Li2010WWW:LinUCB}. 
However, as shown in Fig.~\ref{fig:LinUCB}, the performance of LinUCB is sensitive to the choice of constant $\alpha$.
Actually, for LinUCB, the appropriate choice of constant $\alpha$ depends on the minimum gap between the rewards of optimal and suboptimal arms, which is impractical to obtain beforehand.

\begin{figure*}[thbp]
	\begin{center}
		\begin{minipage}[t]{\textwidth}
			\begin{center}		
				\subfigure[LinUCB with $\alpha = 0.51$]{\includegraphics[angle = 0,height = 0.23\linewidth,width = 0.32\linewidth]{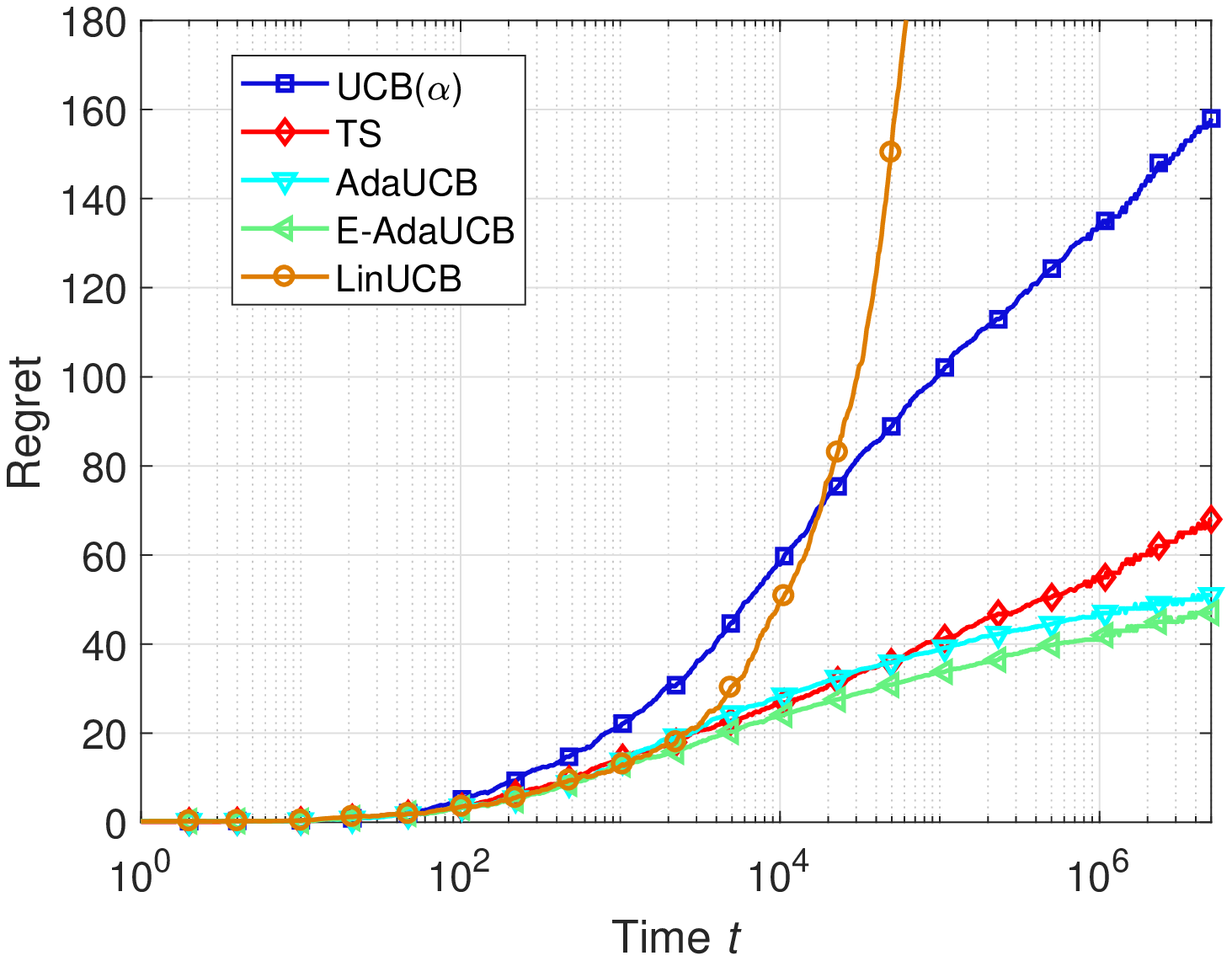}
					\label{fig:linUCB_0p51}}
				\subfigure[LinUCB with $\alpha = 1.0$]{\includegraphics[angle = 0,height = 0.23\linewidth,width = 0.32\linewidth]{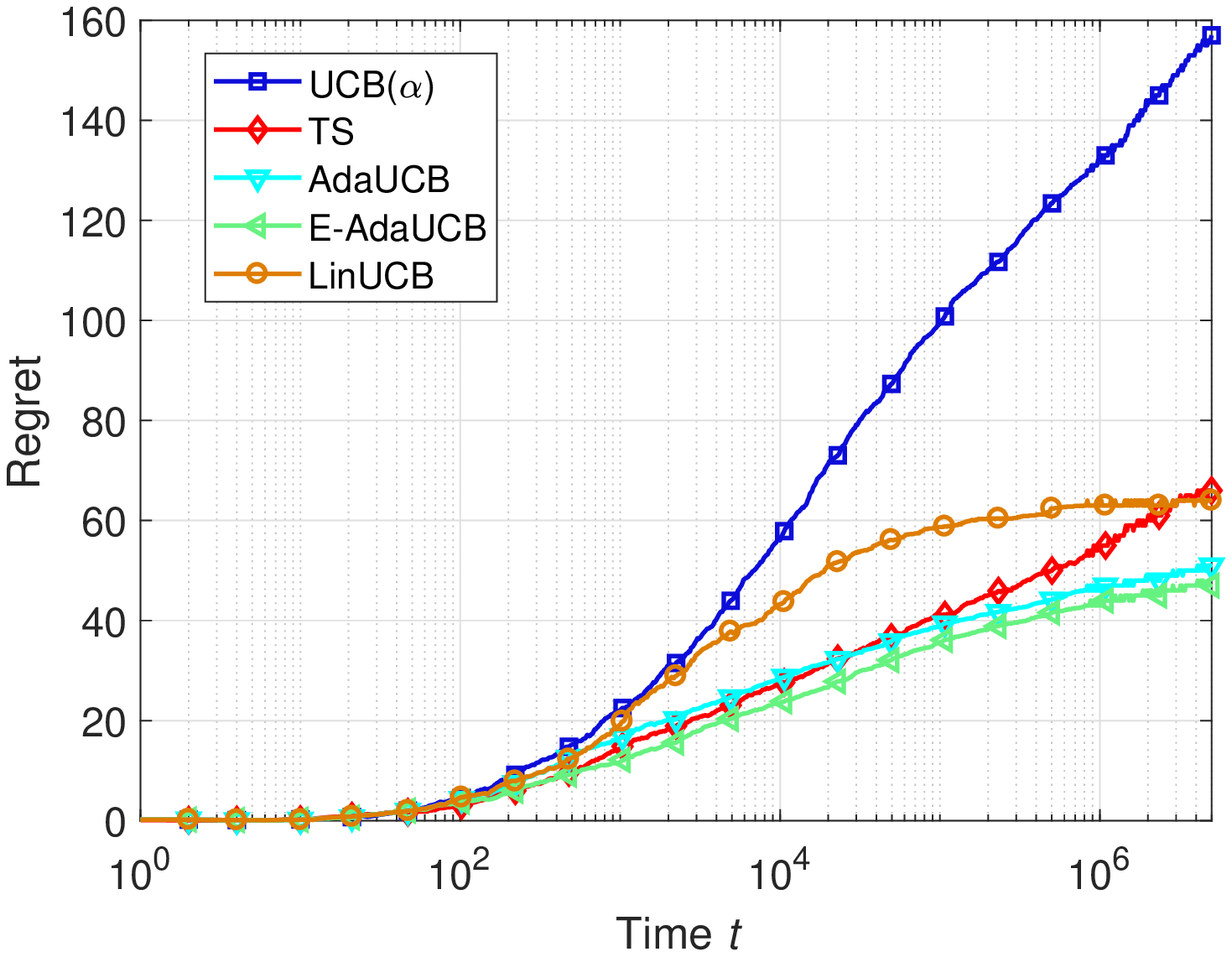}
					\label{fig:linUCB_1p0}}
				\subfigure[LinUCB with $\alpha = 1.2$]{\includegraphics[angle = 0,height = 0.23\linewidth,width = 0.32\linewidth]{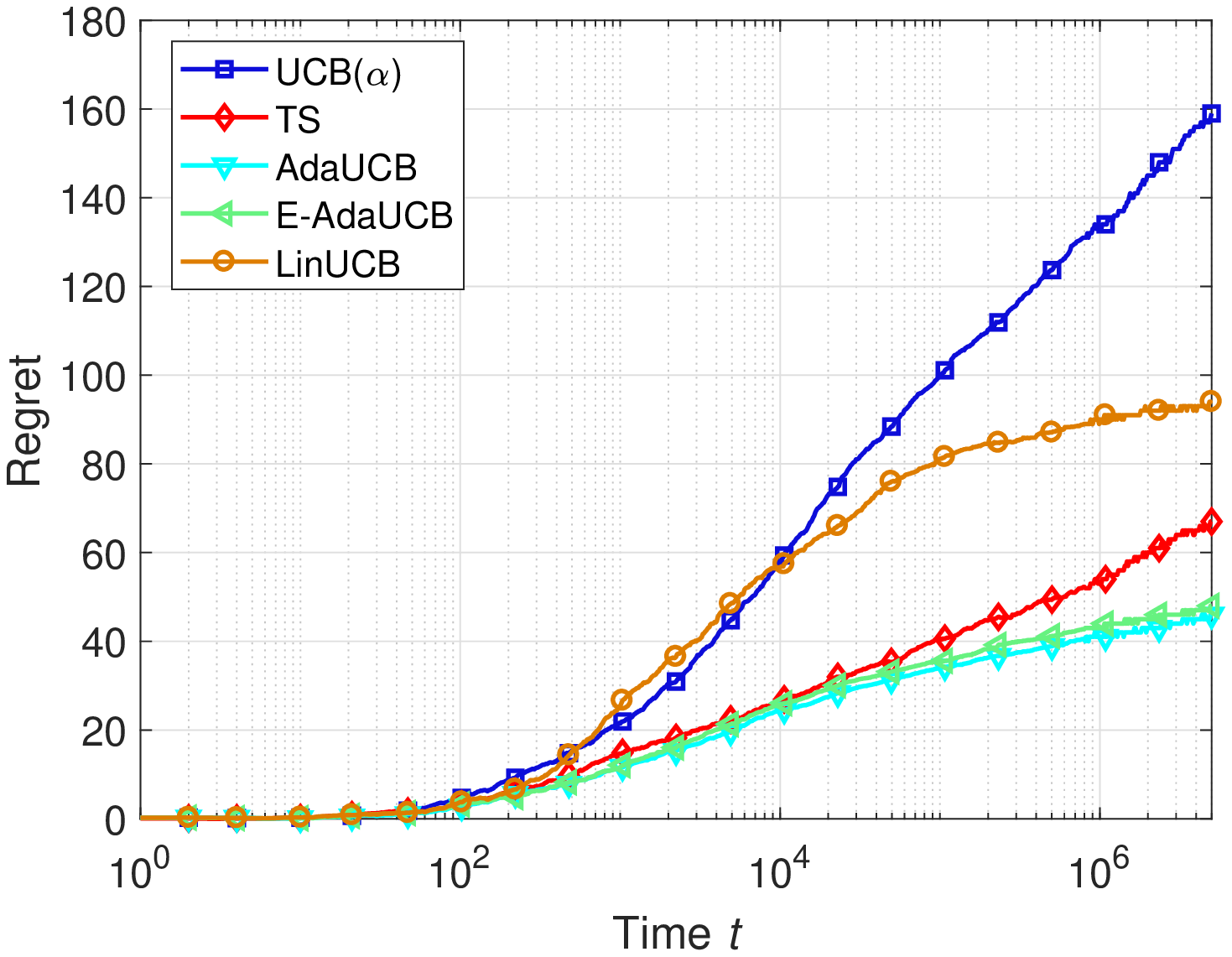}
					\label{fig:linUCB_1p2}}
				\vspace{-.25cm}
				\caption{Regret under beta distributed load with LinUCB algorithms of various constant $\alpha$. (AdaUCB: $l^{(-)} = l^{(-)}_{0.05},  l^{(+)} = l^{(+)}_{0.05}$).}
				\label{fig:LinUCB}
			\end{center}
		\end{minipage}
	\end{center}
\end{figure*}


\end{document}